\documentclass{article}
\usepackage{algorithmic}

\usepackage{etoolbox}
\newtoggle{neurips}
\togglefalse{neurips}
\newcommand{\neurips}[1]{\iftoggle{neurips}{#1}{}}
\newcommand{\arxiv}[1]{\iftoggle{neurips}{}{#1}}

\neurips{
  \PassOptionsToPackage{numbers, compress, square}{natbib}
  \usepackage[]{arXiv-2407.18676v2/neurips_2024}
  \usepackage[numbers]{natbib}
  \bibliographystyle{abbrvnat}
}
\usepackage{caption}
\usepackage{natbib}

\date{}

\arxiv{
\usepackage[letterpaper, left=1in, right=1in, top=1in,
bottom=1in]{geometry}
  \usepackage{parskip}
}

\neurips{
  \usepackage{parskip}
  }

\PassOptionsToPackage{hypertexnames=false}{hyperref}  %

\usepackage[svgnames]{xcolor}
\usepackage[colorlinks=true, linkcolor=blue!70!black, citecolor=blue!70!black,urlcolor=black,breaklinks=true]{hyperref}
\colorlet{txblue}{RoyalBlue!70!NavyBlue}
\hypersetup{linkcolor=txblue,
            citecolor=txblue}
\usepackage{microtype}
\usepackage{hhline}

\usepackage{amsthm}
\usepackage{mathtools}
\usepackage{amsmath}
\usepackage{bbm}
\usepackage{amsfonts}
\usepackage{amssymb}
\usepackage[nameinlink,capitalize]{cleveref}

\makeatletter
\newcommand{\neutralize}[1]{\expandafter\let\csname c@#1\endcsname\count@}
\makeatother

\newtheorem*{lemma*}{Lemma}

\usepackage{algorithm}

\arxiv{
\usepackage{natbib}
\bibliographystyle{plainnat}
\bibpunct{(}{)}{;}{a}{,}{,}
}

\usepackage{xpatch}

\usepackage{thmtools}
\usepackage{thm-restate}
\declaretheorem[name=Theorem,parent=section]{theorem}
\declaretheorem[name=Lemma,parent=section]{lemma}
\declaretheorem[name=Assumption, parent=section]{assumption}
\declaretheorem[name=Condition, parent=section]{condition}

\declaretheorem[name=Remark, parent=section]{remark}
\declaretheorem[name=Proposition, parent=section]{proposition}

\usepackage{crossreftools}
\makeatletter
  \renewenvironment{proof}[1][Proof]%
  {%
   \par\noindent{\bfseries\upshape {#1.}\ }%
  }%
  {\qed\newline}
  \makeatother

\theoremstyle{definition}  %

\newtheorem{corollary}{Corollary}[section]

\theoremstyle{plain}
\newtheorem{definition}{Definition}[section]

\xpatchcmd{\proof}{\itshape}{\normalfont\proofnameformat}{}{}
\newcommand{\proofnameformat}{\bfseries}

\renewcommand{\eqref}[1]{\texorpdfstring{\hyperref[#1]{(\ref*{#1})}}{(\ref*{#1})}}
\crefformat{equation}{#2Eq.\,(#1)#3}
\Crefformat{equation}{#2Eq.\,(#1)#3}

\Crefformat{figure}{#2Figure~#1#3}
\Crefformat{assumption}{#2Assumption~#1#3}

\Crefname{assumption}{Assumption}{Assumptions}

\crefname{fact}{Fact}{Facts}

\Crefformat{figure}{#2Figure #1#3}
\Crefformat{assumption}{#2Assumption #1#3}

\usepackage{crossreftools}
\usepackage{xparse}

\ExplSyntaxOn
\DeclareDocumentCommand{\XDeclarePairedDelimiter}{mm}
 {
  \__egreg_delimiter_clear_keys: %
  \keys_set:nn { egreg/delimiters } { #2 }
  \use:x %
   {
    \exp_not:n {\NewDocumentCommand{#1}{sO{}m} }
     {
      \exp_not:n { \IfBooleanTF{##1} }
       {
        \exp_not:N \egreg_paired_delimiter_expand:nnnn
         { \exp_not:V \l_egreg_delimiter_left_tl }
         { \exp_not:V \l_egreg_delimiter_right_tl }
         { \exp_not:n { ##3 } }
         { \exp_not:V \l_egreg_delimiter_subscript_tl }
       }
       {
        \exp_not:N \egreg_paired_delimiter_fixed:nnnnn 
         { \exp_not:n { ##2 } }
         { \exp_not:V \l_egreg_delimiter_left_tl }
         { \exp_not:V \l_egreg_delimiter_right_tl }
         { \exp_not:n { ##3 } }
         { \exp_not:V \l_egreg_delimiter_subscript_tl }
       }
     }
   }
 }

\keys_define:nn { egreg/delimiters }
 {
  left      .tl_set:N = \l_egreg_delimiter_left_tl,
  right     .tl_set:N = \l_egreg_delimiter_right_tl,
  subscript .tl_set:N = \l_egreg_delimiter_subscript_tl,
 }

\cs_new_protected:Npn \__egreg_delimiter_clear_keys:
 {
  \keys_set:nn { egreg/delimiters } { left=.,right=.,subscript={} }
 }

\cs_new_protected:Npn \egreg_paired_delimiter_expand:nnnn #1 #2 #3 #4
 {%
  \mathopen{}
  \mathclose\c_group_begin_token
   \left#1
   #3
   \group_insert_after:N \c_group_end_token
   \right#2
   \tl_if_empty:nF {#4} { \c_math_subscript_token {#4} }
 }
\cs_new_protected:Npn \egreg_paired_delimiter_fixed:nnnnn #1 #2 #3 #4 #5
 {
  \mathopen{#1#2}#4\mathclose{#1#3}
  \tl_if_empty:nF {#5} { \c_math_subscript_token {#5} }
 }
\ExplSyntaxOff

\XDeclarePairedDelimiter{\supnorm}{
  left=\lVert,
  right=\rVert,
  subscript=\infty
  }

\let\Pr\undefined

\DeclareMathOperator{\Pr}{Pr}

\def\ddefloop#1{\ifx\ddefloop#1\else\ddef{#1}\expandafter\ddefloop\fi}
\def\ddef#1{\expandafter\def\csname bb#1\endcsname{\ensuremath{\mathbb{#1}}}}
\ddefloop ABCDEFGHIJKLMNOPQRSTUVWXYZ\ddefloop
\def\ddefloop#1{\ifx\ddefloop#1\else\ddef{#1}\expandafter\ddefloop\fi}
\def\ddef#1{\expandafter\def\csname b#1\endcsname{\ensuremath{\mathbf{#1}}}}
\ddefloop ABCDEFGHIJKLMNOPQRSTUVWXYZ\ddefloop
\def\ddef#1{\expandafter\def\csname sf#1\endcsname{\ensuremath{\mathsf{#1}}}}
\ddefloop ABCDEFGHIJKLMNOPQRSTUVWXYZ\ddefloop
\def\ddef#1{\expandafter\def\csname c#1\endcsname{\ensuremath{\mathcal{#1}}}}
\ddefloop ABCDEFGHIJKLMNOPQRSTUVWXYZ\ddefloop
\def\ddef#1{\expandafter\def\csname h#1\endcsname{\ensuremath{\widehat{#1}}}}
\ddefloop ABCDEFGHIJKLMNOPQRSTUVWXYZ\ddefloop
\def\ddef#1{\expandafter\def\csname hc#1\endcsname{\ensuremath{\widehat{\mathcal{#1}}}}}
\ddefloop ABCDEFGHIJKLMNOPQRSTUVWXYZ\ddefloop
\def\ddef#1{\expandafter\def\csname t#1\endcsname{\ensuremath{\widetilde{#1}}}}
\ddefloop ABCDEFGHIJKLMNOPQRSTUVWXYZ\ddefloop
\def\ddef#1{\expandafter\def\csname tc#1\endcsname{\ensuremath{\widetilde{\mathcal{#1}}}}}
\ddefloop ABCDEFGHIJKLMNOPQRSTUVWXYZ\ddefloop
\def\ddefloop#1{\ifx\ddefloop#1\else\ddef{#1}\expandafter\ddefloop\fi}
\def\ddef#1{\expandafter\def\csname scr#1\endcsname{\ensuremath{\mathscr{#1}}}}
\ddefloop ABCDEFGHIJKLMNOPQRSTUVWXYZ\ddefloop

\input{widebar}

\usepackage{hyperref}       
\usepackage{url}            
\usepackage{amsmath}        
\usepackage{natbib}         
\usepackage[english]{babel}
\usepackage{graphicx}       
\usepackage{booktabs}
\usepackage{makecell}
\usepackage{hyperref}

\usepackage{math_commands}  
\usepackage{caption}
\usepackage[toc,page,header]{appendix}
\usepackage{minitoc}

\DeclareMathOperator*{\argmax}{arg\,max}

\title{Revisiting Social Welfare in Bandits: UCB is (Nearly) All You Need}

\author{
  Dhruv Sarkar\thanks{Indian Institute of Technology Kharagpur, \texttt{dhruv.sarkar@kgpian.iitkgp.ac.in}} 
  \and
  Nishant Pandey\thanks{Indian Institute of Technology Kanpur, \texttt{nishantp22@iitk.ac.in}} 
  \and
  Sayak Ray Chowdhury\thanks{Indian Institute of Technology Kanpur, \texttt{sayakrc@iitk.ac.in}} 
}

\usepackage{amssymb,amsfonts,amsmath,amsthm} 
\usepackage{amsmath}
\usepackage{amsfonts} 
\usepackage{mathtools} 
\usepackage{bm} 
\usepackage{mathrsfs}
\usepackage{thmtools}
\usepackage{thm-restate}

\usepackage{enumerate}
\usepackage{comment}
\usepackage{multirow}
\usepackage{xcolor}         
\usepackage{xfrac}

\usepackage{subcaption}

\begin{document}

\maketitle

\begin{abstract}

Regret in stochastic multi-armed bandits traditionally measures the difference between the highest reward and either the arithmetic mean of accumulated rewards or the final reward. These conventional metrics often fail to address fairness among agents receiving rewards, particularly in settings where rewards are distributed across a population, such as patients in clinical trials. To address this, a recent body of work has introduced Nash regret, which evaluates performance via the geometric mean of accumulated rewards, aligning with the Nash social welfare function known for satisfying fairness axioms. 

To minimize Nash regret, existing approaches
require specialized algorithm designs and strong assumptions, such as multiplicative concentration inequalities and bounded, non-negative rewards, making them unsuitable for even Gaussian reward distributions. We demonstrate that an initial uniform exploration phase followed by a standard Upper Confidence Bound (UCB) algorithm achieves near-optimal Nash regret, while relying only on additive Hoeffding bounds, and naturally extending to sub-Gaussian rewards. Furthermore, we generalize the algorithm to a broad class of fairness metrics called 
the $p$-mean regret, proving (nearly) optimal regret bounds uniformly across all $p$ values.
This is in contrast to prior work, which made extremely restrictive assumptions on the bandit instances and even then achieved suboptimal regret bounds. Numerical simulations validate our method’s practical efficacy, broadening the accessibility of fairness in bandit algorithms.Our experiments can be reproduced using the following code: \url{https://github.com/NP-Hardest/UCBisAllYouNeed}.
\end{abstract}

\section{Introduction}

The multi-armed bandit (MAB) problem is a foundational framework for sequential decision-making under uncertainty, with applications that span healthcare, advertising, education, and beyond.
In this setup, a decision-maker sequentially selects from a set of arms $[k]:=\{1,2,...k\}$ -- each having reward distribution $\rho_i$ with unknown mean $\mu_i \in \mathbb{R}$ -- across a time horizon $T$, aiming to minimize regret, which quantifies the loss incurred by not always choosing the arm with the highest mean reward $\mu^* = \max_{i \in [k]} \mu_i$. Traditionally, regret is measured as the difference between $\mu^*$ and the arithmetic mean of accumulated rewards $\frac{1}{T}\sum_{t=1}^T \mathbb{E}[\mu_{I_t}] $, where $I_t$ is the arm selected at $t$-th round depending on past (random) observations. This definition of (average) regret often overlooks fairness considerations, particularly in settings where rewards correspond to values accruing to a population of users, such as patients in clinical trials. While average regret equates to maximizing social welfare, as defined by the average reward, it may still permit significant disparities across different users.

To address these fairness limitations, \cite{barman2023fairness} introduce Nash regret—a strengthened notion based on the Nash Social Welfare (NSW) function -- defined using the geometric mean of expected rewards $\left(\prod_{t=1}^T \mathbb{E}[\mu_{I_t}] \right)^\frac{1}{T}$ -- a function well-known in economics for satisfying key fairness axioms \citep{moulin2004fair}. By minimizing Nash regret, bandit algorithms promote equitable outcomes that balance efficiency and fairness, ensuring that users across all rounds benefit fairly from the decision-making process while still enabling exploration to identify optimal arms.

\begin{table*}[t]
        \centering
        \begin{tabular}{ccccc}
            \toprule
            Reference &   $p \in [0, 1] $ & $ p \in [-1, 0)$  &  $p < -1 $ & Assumptions\\
            \midrule
            \citet{barman2023fairness}
            
           & $\widetilde{O}\left(\sqrt{\frac{k}{T}}\right)$ & - & - & Rewards bounded in $[0,1]$ \\
            \midrule
            \citet{krishna2025p}
            
           & \centering $\widetilde{O}\left(\sqrt{\frac{k}{T}}\right)$ & $\widetilde{O}\left(k^{\frac{3}{4}}{T}^{\frac{-1}{4}}\right)$ &  $\widetilde{O}\left(k^{\frac{1}{2}}{T}^{\frac{-1}{4|p|}}\right)$ & \thead{Rewards bounded in $[0,1]$ \\ $\mu_i\geq \frac{32\sqrt{k\log T\sqrt{\log k}}}{T^{1/4}} \ \forall \ i\in [k]$}\\
            \midrule
            Ours
           & \centering $\widetilde{O}\left(\sqrt{\frac{k}{T}}\right)$ & $\widetilde{O}\left(\frac{k^\frac{|p|+1}{2}}{\sqrt T}\right)$ & $\widetilde{O}\left(\frac{|p| k^\frac{|p|+1}{2}}{\sqrt T}\right)$ & \thead{Sub-Gaussian rewards\\ $\mu_i \geq 0 \ \forall \ i\in [k]$}\\
            \bottomrule
        \end{tabular}
        \caption{Summary of regret bounds for different $p \in (-\infty,1]$. Our results not only do away with restrictive assumptions of prior work but also significantly improve the regret bounds to achieve order-optimality.\label{tab:results}}
    \end{table*}
    
While Upper Confidence Bound (UCB, \cite{bubeck2012regret}) is a classical index-based strategy to minimize average regret, \cite{barman2023fairness} argues it doesn't suffice to minimize Nash regret. They propose a new index, Nash Confidence Bound (NCB) $\widehat{\mu}_{i}+\sqrt{\frac{2\hat{\mu}_{i}\log T}{n_{i}}}$, where the confidence width depends on $\widehat \mu_i$--an unbiased estimate of $\mu_i$ from $n_i$ independent samples (e.g. sample mean). To show that NCBs are optimistic estimates of unknown true means, they resort to multiplicative Hoeffding/Chernoff bounds, which put a rather restrictive assumption on reward distributions--each $p_i$ to have support on a non-negative interval in $\mathbb{R}$. It also implicitly puts a cap on the value of $\mu^*$ and requires their algorithm to have knowledge of that. Instead, it is desirable to have algorithms that work under generic reward distributions (e.g., Gaussian) and don't require any upper bound on $\mu^*$.

\citet{krishna2025p} study a more general class of fairness metric, the $p$-mean welfare -- defined using the generalized power mean of expected rewards $\left(\frac{1}{T}\sum_{t=1}^T (\mathbb{E}[\mu_{I_t}])^p\right)^\frac{1}{p}$ -- a function with roots in social choice theory \citep{moulin2004fair}. By changing the value of the parameter $p \in (-\infty,1]$, this single function can be made to behave like a utilitarian function (focused on maximizing total utility), a Rawlsian function (focused on the worst-off individual), or an intermediate function like Nash social welfare. This generalized metric thus helps us study all the social objectives under a single umbrella, based on the value of $p$.

To minimize $p$-mean regret (which includes Nash regret as a special case), \citet{krishna2025p} employs the UCB index. Their results depend on a restrictive assumption that every arm has an expected reward of at least the order $\sqrt{k}/T^{\frac{1}{4}}$. This is in stark contrast to a counter-example presented in~\citet{barman2023fairness} that demonstrates the failure of UCB in minimizing Nash regret. It considers two Bernoulli arms with means \(\mu_1 = (2e)^{-T}\) and \(\mu_2 = 1\). The assumption of the minimum expected reward excludes this example, significantly limiting the applicability of their algorithm to bandit instances like this. Moreover, they assume bounded and non-negative rewards, as in \citet{barman2023fairness}, further narrowing the scope.

We address these shortcomings in prior work by introducing a systematic framework to minimize both Nash and $p$-mean regrets. Our study demonstrates that the UCB index, combined with a data-adaptive initial exploration step, is sufficient to achieve non-trivial, nearly optimal regret bounds without any restrictive assumptions on the reward distributions and their (unknown) means. More specifically, we make the following contributions:

\begin{enumerate}
\item We introduce a reduction framework that enables us to minimize Nash regret via a short adaptive uniform exploration phase, followed by the execution of a standard bandit algorithm, such as UCB. Our data-adaptive stopping rule for exploration is the key to facilitating this reduction, thereby demonstrating the versatility of the UCB index in minimizing Nash regret. Moreover, the reduction seamlessly adapts to all possible values of the fairness parameter $p$ (for Nash regret, $p=0$). We utilize this insight to design Welfarist-UCB, a novel bandit algorithm that minimizes both Nash and $p$-mean regrets.

\item To bound Nash/$p$-mean regret of Welfarist-UCB, we work with the (additive) Hoeffding inequality instead of the multiplicative one (which is used in~\citet{barman2023fairness}).
This helps us sidestep restrictive (e.g., bounded, non-negative) assumptions on the rewards that multiplicative bounds often require. Notably, multiplicative bounds tend to be inapplicable or significantly looser in broader settings, such as sub-Gaussian distributions. By relying on additive Hoeffding bounds, our algorithm naturally works under sub-Gaussian rewards and also doesn't require any upper bound on the optimal reward $\mu^*$.

\item (a) We prove that Welfarist-UCB attains $\widetilde O\Big(\sigma \sqrt{\frac{k}{T}} \Big)$ upper bound on Nash regret for $\sigma$-sub-Gaussian rewards. This bound is order-optimal and includes the result of \citet{barman2023fairness} for bounded, non-negative rewards as a special case. 

(b) For $p \in [0,1]$, we obtain an order-optimal $p$-mean regret of $\widetilde O\Big(\sigma \sqrt{\frac{k}{T}} \Big)$ for Welfarist-UCB, which not only includes the bound of \citet{krishna2025p} for bounded, non-negative rewards as a sub-case but also gets rid of their unrealistic assumption of each $\mu_i \ge \widetilde \Omega\big( \sqrt{k}T^{-1/4}\big)$.

(c) For any $p < 0$, we prove that $p$-mean regret of Welfarist-UCB is $\widetilde{O}\Big(\frac{\sigma k^\frac{|p|+1}{2}}{\sqrt T}\cdot\max\{1,|p|\}\Big)$. When $p \ge -1$, the regret can be further bounded by $\widetilde O \big(k/\sqrt{T}\big)$, which is tighter than the $\widetilde{O}\big(k^{3/4}T^{-\frac{1}{4}}\big)$ bound of~\citet{krishna2025p} since $T > k$. 

(d) As $p$ becomes more negative (e.g., $p < -1$), our regret bound grows exponentially w.r.t. $|p|$ due to stricter fairness requirements while keeping the asymptotic scale w.r.t. time horizon $T$ same. In the extreme case when $p \to -\infty$ (i.e., when the regret is Rawlsian), the bound becomes vacuous unless $T > O\big(p^2 k^{|p|}\big)$, essentially highlighting a ``no-free-lunch" principle in $p$-mean regret minimization. \citet{krishna2025p} avoids the exponential scaling w.r.t. $p$ 
at the expense of a worse $T^{-\frac{1}{4|p|}}$
scaling w.r.t. $T$ and resorting to a restrictive assumption of $\mu_i \ge \widetilde \Omega\big( \frac{\sqrt{k}}{T^{1/4}}\big)$ for all $i$. 

(e) We validate the practical effectiveness of our approach through numerical simulations, comparing our algorithm with prior work across different values of the fairness parameter $p$ and demonstrating its utility in fairness-aware sequential decision-making. The theoretical results and comparisons are summarized in Table~\ref{tab:results}. 
\end{enumerate}
A comprehensive survey on other related works is presented in Section~\ref{sec:related_work} of the Appendix.

\section{Preliminaries}

In the stochastic multi-armed bandit setting,
the learner (algorithm) has access to $k$ probability distributions $\rho_1,\ldots,\rho_k$ (referred to as arms). Upon pulling an arm $i \in [k]:=\lbrace 1,\ldots, k \rbrace$, the learner observes a (random) reward $R_i$ sampled independently from $\rho_i$. We assume that each $R_i$ is sub-gaussian with mean $\mu_i$ and variance-proxy $\sigma^2$, i.e.,
\begin{align*}
  \mathbb{E}\left[ R_i \right] = \mu_i~,\,   \mathbb{E}\left[  \exp \left(\lambda(R_i - \mu_i)\right)\right] \le \exp\left(\frac{\sigma^2 \lambda^2}{2}\right)~.
\end{align*}
We assume that the mean rewards are non-negative, i.e., $\mu_i \ge 0$ for all arms $i \in [k]$. 
\begin{remark}
This assumption is standard in literature~\citep{barman2023fairness,sawarni2024nash,krishna2025p}, motivated by social welfare applications such as clinical trials or resource allocation, where candidate arms are typically pre-screened to be non-harmful on average. The goal is thus to identify the most beneficial option among safe alternatives. Moreover, fairness metrics like Nash Social Welfare, based on the geometric mean, can only be defined for non-negative means. However, this does not preclude observing negative individual rewards: for example, a treatment may be beneficial on average ($\mu_i \ge 0$) yet can cause negative side effects in some instances. 
Our use of additive concentration bounds allows the model to handle these individual negative outcomes gracefully, even while the assumption of non-negative means remains essential for the coherence of the welfare functions and the subsequent regret analysis.
\end{remark}
The learning process unfolds over $T \geq 1$ rounds, where each round corresponds to a (distinct) user. At each round $t \in [T]$, the algorithm pulls an arm $I_t \in [k]$ and observes a reward $R_t \sim \rho_{I_t}$, where 
$I_t$ depends on the arm pulls and (random) observed rewards till round $t-1$. 
The sequence of expected rewards $\mathbb{E}[\mu_{I_t}], t \in [T]$, can be mapped to a social welfare value using a function $f:\mathbb{R}^T \to \mathbb{R}$ and then subtracted from the maximum welfare $\mu^*$ to quantify the algorithm's performance (regret) across $T$ users. 
When $f$ returns the arithmetic mean of expected rewards, we get the average regret
$\ARg_T := \mu^* - \frac{1}{T} \sum\nolimits_{t=1}^T \mathbb{E}[\mu_{I_t}]$.
When $f$ returns the geometric mean of expected rewards, we get the Nash regret
\begin{equation*}\label{eq:nash-reg}
    \NRg_T := \mu^* - \left(\prod\nolimits_{t=1}^T \mathbb{E}[\mu_{I_t}] \right)^\frac{1}{T}~.
\end{equation*}
By AM-GM inequality, $\NRg_T \ge \ARg_T$, and thus minimizing Nash regret is harder compared to average regret.
When $f$ returns the generalized power mean of expected rewards, we get, for $p \in \mathbb{R}$, the $p$-mean regret 
\begin{equation*}\label{eq:pmean-reg}
    \Rg^p_T := \mu^* - \left(\frac{1}{T}\sum\nolimits_{t=1}^T (\mathbb{E}[\mu_{I_t}])^p\right)^\frac{1}{p} \, .
\end{equation*}
Unlike average and Nash regrets, the $p$-mean regret captures the complete spectrum of interplay between fairness and utility.

For $p > 1$, the function $f$ (generalized power mean) puts more emphasis on rounds with high rewards, reflecting a more utilitarian perspective. In the extreme case when $p \to \infty$, $f$ returns the maximum of expected rewards $\max_{t=1}^T \mathbb{E}[\mu_{I_t}]$, which focuses on the best-off individual only and does not provide any fairness guarantee. For $p=1$, $p$-mean regret coincides with the average regret, focusing on the average utility across $T$ individuals. For $p < 1$, the focus shifts to fairness. The function $f$ satisfies the Pigou-Dalton transfer principle, which ensures that transferring a small amount of reward from a well-off individual to another one with lower utility increases the overall welfare~\citep{moulin2004fair}. For $p=0$, $p$-mean regret coincides with the Nash regret, focusing on average welfare across $T$ individuals. As $p$ decreases further, the function puts more emphasis on rounds with low rewards. In the extreme case when $p \to -\infty$, $f$ returns the minimum of expected rewards $\min_{t=1}^T \mathbb{E}[\mu_{I_t}]$, which focuses on the worst-off individual only and does not provide any utility guarantee.
In summary, $p$ can be viewed as a parameter trading off fairness and utility, with the region of interest $p \le 1$, where a smaller value of $p$ ensures more fairness and vice versa.

\section{Algorithm and Results}
\label{sec:secAlgo}


\begin{algorithm}[t]
    \caption{Welfarist UCB }
    \label{algo:ucb:modified}
    \noindent
    \textbf{Input:} Number of arms $k$, time horizon $T$, fairness measure $p$ and reward variance-proxy $\sigma^2$.\\
    \vspace{-10pt}
    \begin{algorithmic}[1]
        \STATE Initialize empirical means $\widehat{\mu}_i = 0$ and counts $n_i = 0$ for all $i \in [k]$, round index $t=1$.
        \IF{$p\geq-1$}
        \STATE Set $p \gets1$.
        \ENDIF
        \\ \texttt{Phase I}
        \WHILE{for all $i \in [k]$, $\widehat{\mu}_i \leq 2\sqrt{\frac{2\sigma^2\log T}{n_i}} \textbf{ or } n_i   \leq 192p^2 \sigma^2\frac{\log T}{\big(\widehat{\mu}_{i} - 2\sqrt{\frac{2\sigma^2\log T}{n_i}}\big)^2}$ 
        } \label{step:PhaseOneAlgTwo}
        \IF{$t\, \text{mod}\, k = 1$}
        \STATE Draw a permutation $\pi$ uniformly at random from the set of all permutations of $[k]$. 
        \ENDIF
            \STATE Pull arm $I_t = \pi \left(1+ (t-1)\,\text{mod}\, k\right)$. 
            \STATE Observe reward $R_t \sim \rho_{I_t}$. 
            \STATE Update $n_{I_t} \gets n_{I_t}+1$, $\widehat{\mu}_{I_t} \gets \frac{(n_{I_t}-1)\widehat{\mu}_{I_t}}{n_{I_t}}+\frac{R_t}{n_{I_t}}$.
            \STATE Update $t\gets t + 1$. 
        \ENDWHILE
        \\ \texttt{Phase II}
        \WHILE{ $t \leq T$}
        \STATE Pull arm $I_t = \argmax_{i \in [k]}  \left\lbrace \widehat{\mu}_i+2\sqrt{\frac{2\sigma^2\log T}{n_i}} \right \rbrace$.
        \STATE Observe reward $R_t \sim \rho_{I_t}$.
            \STATE Update $n_{I_t} \gets n_{I_t}+1$, $\widehat{\mu}_{I_t} \gets \frac{(n_{I_t}-1)\widehat{\mu}_{I_t}}{n_{I_t}}+\frac{R_t}{n_{I_t}}$.
        \STATE Update $t \gets t + 1$.
        \ENDWHILE
    \end{algorithmic}
\end{algorithm}
In this section, we introduce our algorithm (Welfarist-UCB) that decomposes Nash/$p-$mean regret minimization into two phases: (I) Uniform exploration over a data-adaptive horizon, and (II) Explore-exploit optimization using the UCB index.

\paragraph{Phase I (Uniform Exploration):} 

In this phase, at the beginning of each block of 
$k$ steps, we draw a uniform random permutation 
$\pi \in \Pi_k$ of the arms, where $\Pi_k$ denotes the set of all $k!$ permutations of $\lbrace 1,2,\ldots,k \rbrace$. We then select the arms sequentially in the order $\pi(1), \pi(2),\ldots, \pi(k)$. After $k$ steps, the permutation is exhausted, and a new independent uniform permutation is drawn for the next block.

Phase I continues until the accumulated reward of some arm $i$ exceeds an adaptive threshold.  
Formally, termination occurs at the first time $t$ such that for some arm $i\in[k]$, $\widehat{\mu}_i > 2\sqrt{\tfrac{2\sigma^2\log T}{n_i}}$ and
\begin{align*}
  n_i   > \frac{192p^2 \sigma^2\log T}{\big(\widehat{\mu}_{i} - 2\sqrt{\frac{2\sigma^2\log T}{n_i}}\big)^2}.  
\end{align*}
Here $n_i$ denotes the number of pulls of arm $i$ up to time $t$, and $\widehat{\mu}_i$ denotes the empirical mean of its observed rewards. $\sigma^2$ is the variance-proxy of sub-Gaussian rewards, and $p$ is the fairness parameter provided as input. We normalize it by setting $p=1$ whenever $p\ge -1$, since for $p\geq-1$ case, the analysis for Phase I termination condition does not depend on $p$ (the careful choice of numeric inequalities in our analysis for $p\in[-1,0)$ does away with the involvement of $p$, whereas for $p>0$, the generalised mean inequality trivially helps us avoid $p$). The second terminating condition is critical because it ensures, with high probability, that Phase I ends after $\Theta\big(\frac{1}{(\mu^*)^2} \big)$ rounds, which, in turn, enables us to work with the UCB index in Phase II (see Remark~\ref{rem:terminate}). 
The additional constraint $\widehat{\mu}_i > 2\sqrt{\tfrac{2\sigma^2\log T}{n_i}}$ ensures the condition is meaningful, ruling out cases where the denominator in the threshold expression is non-positive.


\paragraph{Phase II (Explore-exploit with UCB):} In this phase, we employ the UCB index $\widehat{\mu}_i + 2\sqrt{\tfrac{2\sigma^2\log T}{n_i}}$ to pull arms for $\sigma$-sub-Gaussian rewards. For comparison, \citet{barman2023fairness} employs the Nash Confidence Bound (NCB) index
$\hat{\mu}_i + 4\sqrt{\tfrac{\widehat{\mu}_i \log T}{n_i}}$ for bounded $[0,1]$ rewards to minimize Nash regret. Our results show that the simpler UCB rule, along with the data-adaptive uniform exploration phase, suffices to achieve order-optimal bounds on Nash regret. Moreover, this strategy achieves the best known bounds for $p-$mean regret across different regimes of the fairness parameter $p$. The pseudo-code of the strategy (Welfarist-UCB) is 
presented in Algorithm~\ref{algo:ucb:modified}.

\begin{restatable}{theorem}{TheoremImprovedNashRegret}\textup{(Nash Regret of Welfarist-UCB)}
\label{theorem:improvedNashRegret}
For any bandit instance with $k$ arms, each with $\sigma-$sub-gaussian rewards, and given any (moderately large) $T$, Welfarist-UCB achieves a Nash regret 
\begin{align*}
\NRg_T = O \left(\sigma \sqrt{ \frac{k  \log T\log k }{T} } \right).
\end{align*}
\end{restatable}
\begin{remark}[Comparison with NCB~\citep{barman2023fairness}] Theorem~\ref{theorem:improvedNashRegret} establishes a strict generalization of the regret bound achieved by the standard NCB algorithm. In contrast to NCB, which relies on restrictive assumptions requiring rewards to be bounded and non-negative, Welfarist-UCB applies to arbitrary sub-Gaussian rewards. Furthermore, it does not assume any prior upper bound on $\mu^*$. These improvements are enabled by the data-adaptive termination rule for Phase I and the UCB-based arm selection in Phase II, both of which play equally crucial roles (see Remark~\ref{rem:terminate}). A modified NCB algorithm achieves a tighter bound by removing the $\sqrt{\log k}$ factor, but it requires a much larger value of $T$, while $T \ge k$ suffices for us.
\end{remark}

\begin{restatable}{theorem}{TheoremPLTZero}\textup{($p-$mean regret of Welfarist-UCB)}
\label{thm:negative-regret}
Consider a $ k$-armed bandit problem with $\sigma-$sub-gaussian rewards, time horizon $T$, and fairness parameter $p\in \mathbb{R}$. Then, the $p$-mean regret of Welfarist-UCB satisfies  
\begin{align*}
\mathrm{R}^{p}_T & =
\begin{cases}
    O\left(\frac{\sigma k^\frac{|p|+1}{2}\sqrt{\log T}}{\sqrt T}\cdot\max\{1,|p|\}\right), & \text{$p<0$}~,\\
    O\left(\frac{\sigma \sqrt{k\log T\log k }}{\sqrt T}\right), & \text{$p\geq0$}~.         
\end{cases}
\end{align*}
\end{restatable}

\begin{remark}[Lower bound] 
Observe that for $p\leq 1$, $\left(\frac{\sum_{t=1}^T(\E[\mu_{I_t}])^p}{T}\right)^\frac{1}{p} \le \frac{1}{T}\sum_{t=1}^T\E[\mu_{I_t}]$ since the generalized mean is strictly increasing in $p$. Hence, from the standard lower bound on average regret $\ARg_T$ (see, e.g., \citet{bubeck2012regret}), we can lower bound $p$-mean regret as $\Rg^{p}_T \geq \ARg_T \geq \widetilde{\Omega}\big(\sqrt{k/T}\big)$. Thus, our upper bound is optimal up to poly-log factors for $p \in [0,1]$ and is only off by a factor at most $\sqrt{k}$ for $p \in [-1,0)$.
\end{remark}

\begin{remark}[Comparison with Explore-then-UCB~\citep{krishna2025p}]
Theorem~\ref{thm:negative-regret} provides several improvements over Explore-then-UCB. 
First, for $p \ge 0$ (as in the case of NCB), it offers a strict generalization by extending the analysis from bounded, non-negative rewards to arbitrary sub-Gaussian rewards. 
Second, for $p \in [-1,0)$, it achieves a worst-case regret bound of $\widetilde{O}\!\left(k/\sqrt{T}\right)$, which is sharper than the $\widetilde{O}\!\left(k^{3/4}T^{-1/4}\right)$ bound attained by Explore-then-UCB. 
Most importantly, it eliminates the restrictive assumption that each unknown mean $\mu_i \ge \widetilde{\Omega}\!\left(\sqrt{k}\,T^{-1/4}\right)$.%
\footnote{The justification that this assumption is equivalent to $\mu_i \ge 0$ as $T \to \infty$ is unsatisfactory: accepting it would amount to claiming that a regret of $O(T^{-1/4})$ is equivalent to the optimal $O(T^{-1/2})$ rate, since both vanish as $T$ grows.}

For $p < -1$, \citet{krishna2025p} obtain a bound 
$\widetilde{O}\!\left(k^{1/2} T^{-1/4|p|}\right)$.
At first glance, their bound appears to scale better with $k$, while ours does with $T$. However, a more careful scaling analysis shows that our bound is strictly tighter. First, note that regret bounds are vacuous whenever they exceed $\mu^*$. Assuming $\mu^* \leq 1$ for parity with \citet{krishna2025p}, we observe that our bound is non-trivial only when $T > k^{|p|+1}$, whereas their bound becomes non-trivial only when $T > k^{2|p|}$. Moreover, their bound is tighter than ours only in the regime $T \;\leq\; k^{\tfrac{2p^{2}}{2|p|-1}}$, but in this regime both bounds are vacuous (and the trivial bound $\mu^* \leq 1$ dominates), since
$\frac{2p^{2}}{2|p|-1} < |p|+1$ for all  $p < -1$. 
Hence, for all practically relevant values of $T$, our bound is tighter than that of \citet{krishna2025p}.

\end{remark}

\begin{remark}[No free fairness]
\label{rem:freeLunch}
For $p<-1$, our upper bound grows exponentially with $|p|$ unless $T > p^2 k^{\frac{|p|}{2}}$, suggesting that very strong fairness may be unattainable in the short term. In the extreme case when $p \to -\infty$, the bound becomes vacuous, aligning with the intuition that achieving vanishing Rawlsian regret is impossible.
This highlights a `no-free-lunch' phenomenon: although the dependence on $T$ remains asymptotically the same, the intrinsic hardness of the learning problem increases as the notion of fairness becomes increasingly stringent.
A natural direction for future work is to formalize this by proving a matching lower bound in the region $p<-1$.
\end{remark}

\section{Proof Techniques}

We start with a common analytical framework that underlies the analysis of both Nash and $p$-mean regrets (Theorems~\ref{theorem:improvedNashRegret} and \ref{thm:negative-regret}, respectively). 

Our analysis will be based on conditioning on a ``good" event $\cE$, under which at every round $t$, each $\mu_i$ lies in the interval $\left[\widehat{\mu_i}-2\sqrt{\frac{2\sigma^2\log T}{n_i}}, \widehat{\mu_i}+2\sqrt{\frac{2\sigma^2\log T}{n_i}}\right]$ \emph{with high probability}, where $n_i$ and $\widehat \mu_i$ are running updates at round $t$. The proof follows from a standard application of the (additive) Hoeffding bound.





We first bound the total length of Phase I.
\begin{restatable}[Rounds of uniform exploration]{lemma}{LemmaTauAnytime}
\label{tau:anytime}
Let $\tau$ be the (random) number of rounds after which Phase I ends. Then, under the event $\cE$, we have
\begin{align*}
32 \ k \Sample \leq \tau \leq 128 \ k  \Sample~, 
\end{align*} 
almost surely, where $\Sample \coloneqq \frac{4p_{a}^2 \sigma^2\log T}{ (\mu^*)^2}$, with
$p_a = 1$ if $p \ge -1$ and $p_a = p$ if $p < -1$.
\end{restatable}
The lemma stems from our carefully constructed terminating condition. One can show that if $\tau\leq32kS$, then each arm $i$ satisfy 
$n_i   < 192p^2 \sigma^2\frac{\log T}{\big(\widehat{\mu}_{i} - 2\sqrt{\frac{2\sigma^2\log T}{n_i}}\big)^2} $, whereas for $\tau \geq128kS$, there exists atleast one arm (specifically the optimal arm) which violates the above. Based on these, one can conclude that Phase I terminates between these time steps.
See Lemma \ref{tau:lower} and \ref{tau:upper} in the appendix for details.

Recall that in Phase I, we sample a uniform random permutation of the arms every $k$ rounds, and then pull them sequentially. By Lemma~\ref{tau:anytime}, this ensures that each arm is pulled $\Theta\big(\frac{1}{(\mu^*)^2} \big)$ times. This procedure is equivalent to sampling arms uniformly at random with $\Pr[I_t = i] = 1/k$ for any arm $i \in [k]$ so that the ex-ante expected reward satisfies $\E\left[\mu_{I_t}\right] \ge \mu^*/k$. In other words, permutation-based sampling is a \emph{static coupling} of sequential uniform sampling: it fixes the entire sampling order in advance, whereas sequential sampling decides the order dynamically (see Lemma~\ref{lem:permute}).

We next ensure that if an arm is pulled in Phase II, then its mean must be close to $\mu^*$, so its contribution to Nash/$p$-mean regret is low.

\begin{restatable}[Near optimality of Phase II arms]{lemma}{LemmaHighMean:anytime}
\label{lem:suboptimal_arms:anytime}
Under the event $\cE$, $\mu_i  \geq \mu^* - 4\sqrt{\frac{2\sigma^2\log T}{T_i  -1}}$ for any arm $i$ that is pulled at least once in Phase II,
where $T_i$ denotes the total number of pulls of arm $i$.
\end{restatable}
\begin{proof}
Fix any arm $i$ that is pulled at least once in Phase II. When arm $i$ was pulled the $T_i$-th time during Phase II, it must have had the highest UCB. In particular, at that round ${\UCB}(i) \geq {\UCB}(i^*) \geq \mu^*$; the last inequality follows from the Hoeffding bound that the UCB index is an overestimate of the mean with high probability. Therefore, we have
$\widehat{\mu}_i + 2\sqrt{\frac{2\sigma^2\log T}{T_i -1 }} \geq \mu^*$, where $\widehat{\mu}_i$ denotes the sample mean of arm $i$ after $T_i-1$ pulls. 
Also, $ \mu_i \geq \widehat{\mu}_i - 2\sqrt{\frac{2\sigma^2\log T}{T_i -1 }}$ under $\cE$. Combining these two together, we have $\mu_i \geq \mu^* - 4\sqrt{\frac{2\sigma^2\log T}{T_i -1 }} $. 
\end{proof}


\subsection{Proof Sketch for Theorem 1}

We split the analysis of Nash regret into two regimes depending on the magnitude of $\mu^*$. When $\mu^* \leq   \frac{40\sigma\sqrt{2k\log T\log k }}{\sqrt T}$, the regret bound holds trivially.
Thus in the sequel we only consider $\mu^* \geq  \frac{40\sigma \sqrt{2k\log T\log k}}{\sqrt T}$. We first compute Nash social welfare in Phase I using the fact that $\E\left[\mu_{I_t}\right] \ge \mu^*/k$ for each round in Phase I.

\begin{lemma}[NSW in Phase I]
\label{lem:phase1}
Suppose Phase I runs for $\overline T$ rounds. If $\mu^* \geq  \frac{40\sigma \sqrt{2k\log T\log k}}{\sqrt T}$, then
\begin{align*}
\bigg(\prod\nolimits_{t=1}^{\overline{T}} \E [\mu_{I_t} ] \bigg)^\frac{1}{T} \geq  \left(\mu^*\right)^{\frac{\overline{T}}{T}} \bigg(1-{\frac{\overline{T}  \log k}{T}}\bigg)~.   
\end{align*}
\end{lemma}
Next, we lower bound $\big(\prod_{t=\overline{T}+ 1}^{T} \E\left[ \mu_{I_t} \right]\big)^\frac{1}{T}$, the Nash social welfare in Phase II, which is further controlled by $\E\Big[\big( \prod_{t=\overline{T}+1}^{T} \mu_{I_t} \big)^\frac{1}{T}\Big]$ using multivariate Jensen's inequality. To bound this, consider the arms that are pulled at least once after the first $\overline{T}$ rounds. Let $\{1,2, \ldots, \ell\}$ be the set of all those arms and $m_i \ge 1$ be the number of times arm $i \in [\ell]$ is pulled after the first $\overline{T}$ rounds.
Then $\E \Big[\big( \prod_{t=\overline{T}+1}^{T} \ \mu_{I_t} \big)^\frac{1}{T} \Big] = \E\left[ \prod_{i=1}^{\ell} \mu_{i}^\frac{m_i}{T} \right]$.
Moreover, since we are conditioning on the good event $\cE$, Lemma  \ref{lem:suboptimal_arms:anytime} applies to each arm $i \in [\ell]$. Hence, we get
$\E\Big[\prod\limits_{i=1}^{\ell}  \mu_{i}^\frac{m_i}{T}  \Big] 
    \geq \E\Big[\prod\limits_{i=1}^{\ell}\left(\mu^* - 2\sqrt{\frac{ 2\sigma^2\log T}{T_i -1 }} \right)^\frac{m_i}{T} \Big]$.
Using $\sum_{i=1}^\ell m_i = T - \overline{T}$, this is further bounded by 
$(\mu^*)^{1-\frac{\overline{T}}{T}} \E\left[\prod\limits_{i=1}^{\ell}\left(1 - \frac{2 \sqrt{2\sigma^2\log T}}{\mu^*\sqrt{T_i -1}} \right)^\frac{m_i}{T}\right]$.  

Let $\xi_i := \frac{2 \sqrt{2\sigma^2\log T}}{\mu^*\sqrt{T_i -1}}$. Now, our sampling strategy in Phase I, along with Lemma~\ref{tau:anytime}, together imply that each arm is pulled at least $\frac{128\sigma^2\log T}{ (\mu^*)^2}$ times during the first $\overline{T}$ rounds. Hence $T_i-1 \geq \frac{512\sigma^2\log T}{ (\mu^*)^2}$ for each arm $i \in [\ell]$, which yields $\xi_i \leq 1/4$. Thus, applying the inequality $(1-x)^{a} \geq  1- 2ax$ for $x \in \left[0, \frac{1}{2}\right], a \ge 0$, we get
\begin{align*}
   \E\Bigg[\!\prod_{i=1}^{\ell} &\mu_{i}^\frac{m_i}{T} \! \Bigg] \!\geq \!
   (\mu^*)^{1-\frac{\overline{T}}{T}}
   \E \left[\prod_{i=1}^{\ell}\!\left(\!1 - \frac{4 m_i}{T \mu^*}\sqrt{\frac{2\sigma^2 \log T}{T_i -1}} \right)\!\right] \geq (\mu^*)^{1-\frac{\overline{T}}{T}} \E \left[\prod_{i=1}^{\ell}\left(1 - \frac{4}{T\mu^*}\sqrt{ 2m_i\sigma^2 \log T} \right)\right]~,
\end{align*}
where we have used $T_i \ge m_i+1$. Now, further simplification with a Cauchy-Schwarz inequality together with $\sum_{i=1}^l m_i \le T$ and $\ell \le k$
yields a bound on NSW.

\begin{lemma}[NSW in Phase II]
\label{lem:phase2}
Suppose Phase I runs for $\overline T$ rounds. If $\mu^* \geq  \frac{40\sigma \sqrt{2k\log T\log k}}{\sqrt T}$, then
\begin{align*}
\bigg(\!\prod\nolimits_{t=\overline{T}+1}^{T} \E [\mu_{I_t} ] \bigg)^\frac{1}{T} \!\!\geq  \left(\mu^*\right)^{1-\frac{\overline{T}}{T}} \!\bigg(\! 1 -\frac{4}{\mu^*}\sqrt{\frac{ 2k \sigma^2\log T }{T}}\bigg).   
\end{align*}
\end{lemma}


Finally, Lemma~\ref{lem:phase1} and \ref{lem:phase2} together with the high probability bound on the good event $\cE$ complete the proof of Theorem~\ref{theorem:improvedNashRegret}. Refer to the appendix for details.

\begin{remark}[UCB versus NCB]
\label{rem:terminate}
Observe that our terminating condition and permutation-based sampling strategy together ensure that each arm is pulled at least $n_i \approx \widetilde \Omega \big(\frac{1}{(\mu^*)^2} \big)$ times in phase I. This helps us bound $\xi_i \approx \widetilde O \big(\frac{1}{\mu^*\sqrt{n_i}}\big)$ for any arm $i$ that has been pulled at least once in Phase II, with a constant less than $1/2$, which is crucial to control the Nash regret.


In contrast, algorithms in~\cite{barman2023fairness} pull each arm at least either $\widetilde \Omega \big(\sqrt{T} \big)$ or $\widetilde \Omega \big(\frac{1}{\mu^*} \big)$ times in phase I. In both of these cases, $\xi_i \approx \widetilde O \big(\frac{1}{\mu^*\sqrt{n_i}}\big)$ can be bounded by a constant only if one assumes $\mu^* \approx \widetilde \Omega\big(\frac{1}{T^{1/4}}\big)$ or $\widetilde \Omega(1)$, respectively, both of which render the regret under complementary cases sub-optimal. Instead, they resort to NCB-based arm selection in phase II, which, roughly, requires them to bound $\widetilde O \big(\frac{1}{\sqrt{\mu^* n_i}}\big)$ by $1/2$ for controlling Nash regret. This is achieved by assuming $\mu^* \approx \widetilde \Omega\big(\frac{1}{\sqrt{T}}\big)$ since the optimal regret is trivially attained when $\mu^* \approx  \widetilde O\big(\frac{1}{\sqrt{T}}\big)$.\footnote{\citet{krishna2025p} could work with the UCB index as they sidestep this issue by enforcing an unrealistic assumption of $\mu_i \approx  \widetilde \Omega \big(\frac{1}{T^{1/4}}\big)$ for each arm $i$.}

The use of NCB instead of UCB as arm selection index compels them to condition their analysis on the event $\left\lbrace\forall i\in [k],\mu_i \in \!\left[\widehat{\mu_i}-4\sqrt{\frac{\widehat \mu_i\log T}{n_i}}, \widehat{\mu_i}+4\sqrt{\frac{\widehat\mu_i\log T}{n_i}}\right]\!\right\rbrace$ and employ the multiplicative Hoeffding inequality (instead of the additive one) to ensure that it is a ``good" event. However, it restricts their approach to bounded and non-negative rewards only, whereas our approach works for sub-Gaussian rewards because of the use of UCB index and additive Hoeffding bound.

\end{remark}

\subsection{Proof Sketch for Theorem \ref{thm:negative-regret}}


We begin by splitting the analysis into two cases for $p$, namely $p\ge 0$ and $p<0$. We assume that the good event $\cE$ holds with high probability.




\textbf{Case I ($p \ge 0$):} By monotonicity of power means, $\big(\prod_{t=1}^T \E[\mu_{I_t}]\big)^{1/T} \!\!\le\! \big(\frac{1}{T} \sum_{t=1}^T (\E[\mu_{I_t}])^p\big)^{1/p}\!$ for any $p \ge 0$ (see Lemma~\ref{lem:genMeanIneq}). This implies that the $p$-mean regret is at most as large as the Nash regret, and hence, from Theorem~\ref{theorem:improvedNashRegret}, we get
$\Rg_T^p = O\!\left(\sigma \sqrt{\frac{k \log T \log k}{T}}\right)$.
\textbf{Case II ($p<0$):} The analysis further splits into two regimes depending on the magnitude of $\mu^*$. When 
$\mu^* \le O\Big(\frac{\sigma |p| k^{(|p|+1)/2} \sqrt{\log T}}{\sqrt{T}}\Big)$, the regret bound holds trivially. Thus, we focus on the case when $\mu^* \ge \Omega\Big(\frac{\sigma |p| k^{(|p|+1)/2} \sqrt{\log T}}{\sqrt{T}}\Big)$.
First, we set $q=-p>0$ to convert the $p$-mean regret into the $q$-regret
\begin{align*}
    \mathrm{R}^{q}_T \triangleq \mu^* - \bigg(\frac{T}{\sum_{t=1}^T (\E[\mu_{I_t}])^{-q}}\bigg)^\frac{1}{q}~.
\end{align*}
Next, we split the sum into  Phase I and II, and define
\[
x = \frac{T}{\sum_{t=1}^{\overline{T}} (\E[\mu_{I_t}])^{-q}}, \quad
y = \frac{T}{\sum_{t=\overline{T}+1}^{T} (\E[\mu_{I_t}])^{-q}}~,
\]
to obtain $\mathrm{R}_T^q = \mu^* - \big(\frac{1}{1/x + 1/y}\big)^{1/q}$. Thus, our goal would be to get bounds on $1/x$ and $1/y$. To bound $1/x$, note that in Phase I, uniform exploration ensures $\E[\mu_{I_t}] \ge \mu^*/k$, yielding
$\frac{1}{x} \le \frac{\overline{T} k^q}{(\mu^*)^q T}$.

To bound $1/y$, first note that $[\ell]$ denotes the set of arms that are pulled at least once in Phase II and $m_i$ denotes the number of times arm $i \in [\ell]$ is pulled in Phase II. Then an application of Jensen's inequality gives $\sum_{t=\overline{T}+1}^{T} (\E[\mu_{I_t}])^{-q} = \E \left[\sum_{i=1}^\ell m_i \mu_i^{-q} \right]$. Further applying Lemma~\ref{lem:suboptimal_arms:anytime} to each arm $i \in [\ell]$, we get $\frac{1}{y} \le \frac{\E\big[\sum_{i=1}^\ell m_i (\mu^*-u_i)^{-q} \big]}{T}$, where $u_i = 4 \sqrt{\frac{2\sigma^2 \log T}{T_i-1} }$.

Next, we control the terms $(\mu^* - u_i)^{-q}$. Note that for sufficiently small $u_i/\mu^*$, one can linearize $(1-u_i/\mu^*)^{-q}$ as $ 1 + \frac{2qu_i}{\mu^*}$. With the choice of $\mu^*$ and the lower bound $T_i-1 \geq \frac{128 p_a^2\sigma^2\log T}{ (\mu^*)^2}$ from Lemma~\ref{tau:anytime}, we argue that $u_i/\mu^*$ is sufficiently small to apply the above linearization. Thus, the inverse terms of the form $(\mu^*-u_i)^{-q}$ in the $q$-regret can be bounded linearly in $q u_i/\mu^*$, allowing us to control the contributions of suboptimal arms $i \in [l]$ in the $p$-mean regret. We use two slightly different variations of the above linearization for the two regimes $p\leq-1$ and $-1<p<0$ (see Claims \ref{lem:binomial_second} and \ref{lem:binomial_third} in the Appendix); however, the final result remains the same for both cases.


The reason behind linearizing $(\mu^* - u_i)^{-q}$ is that after some simplification, we can bound each of the terms $ m_i (\mu^*-u_i)^{-q}$ by $(\mu^*)^{-q} \left(m_i +\frac{4q\sqrt{m_i}}{\mu^*}\sqrt{{2\sigma^2 \log T}}\right)$. The $\sqrt{m_i}$ terms help us apply Cauchy-Schwarz inequality to handle $\sum_{i \in [\ell]} \sqrt{m_i}$, since $ \sum_{i \in [\ell]}m_i\leq T-\overline{T}\leq T$, which helps us deal only with $T$. Using these steps and making further simplifications, we get 
\begin{align*}
\frac{1}{y} \leq \frac{1}{(\mu^*)^q\Big(1-\frac{4q\sqrt{2k\sigma^2 \log T}}{\mu^* \sqrt{T}}\Big)}~.
\end{align*}
Substituting the bounds on $1/x$ and $1/y$ in the denominator for the $q-$regret, and simplifying this further using the standard numeric inequality $(1+a)^{-1} \geq 1-a$ for $a \in [0,1)$, we get
\begin{align*}
    &\frac{1}{\frac{1}{x}+\frac{1}{y}}\geq (\mu^*)^q \bigg(1-\frac{4q\sqrt{2k\sigma^2\log T}}{\mu^*\sqrt T}-\frac{\overline{T}k^q}{T}\bigg) \geq (\mu^*)^q \bigg(1-\frac{4q\sqrt{2k\sigma^2\log T}}{\mu^*\sqrt T}-\frac{512q^2k^{q+1}\sigma^2 \log T}{(\mu^*)^2 T}\bigg)~,
\end{align*}
where the last inequality follows from
Lemma~\ref{tau:anytime} since $\overline{T} \le 128kS = \frac{512 k q^2 \sigma^2 \log T}{(\mu^*)^2}$. Defining $v$ as the sum of the two negative terms inside the parentheses, we can show that $v \le 1/2$ in the working regime of $\mu^*$. Then, exponentiating both sides by $\frac{1}{q}$ and applying the binomial approximation $(1-v)^{1/q} \ge 1 - \frac{2v}{q}$ for $v \in [0,1/2]$, we obtain 
\begin{align*}
\Big(\!\frac{1}{\frac{1}{x} + \frac{1}{y}}\!\Big)^{\tfrac{1}{q}}\!\geq  \mu^*\!-\frac{8\sqrt{2k\sigma^2\log T}}{\ \sqrt T} -\frac{1024qk^{q+1}\sigma^2 \log T}{\mu^* T}~.
\end{align*}
Substituting the assumed regime for $\mu^*$, we get the desired terms of the form $\frac{k^{\frac{q+1}{2}}}{\sqrt{T}}$. The final step is to use this find upper bound the regret as $\mathrm{R}_T^q = \mu^* - \left(\frac{1}{\frac{1}{x} + \frac{1}{y}}\right)^{\frac{1}{q}}$, which gives us $\mathrm{R}_T^q \le O\Big(\frac{\sigma  k^{(|p|+1)/2} \sqrt{\log T}}{\sqrt{T}}\Big)$. Combining the two results from the two cases on $\mu^*$, we get the stated bound on the $p$-mean regret. Please refer to Appendix \ref{app:pmeanproof} for the complete proof.
\begin{figure*}[t]\centering
		\begin{subfigure}[b]{0.32\textwidth}
			\includegraphics[scale=0.28]{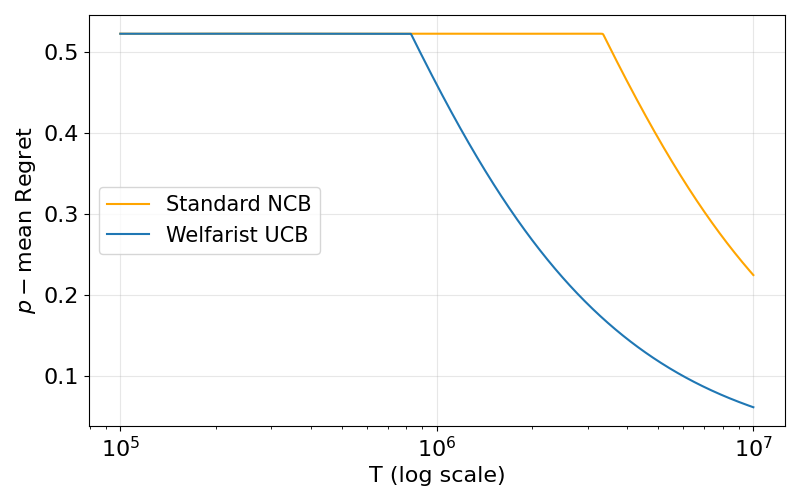}
			\caption{\centering Standard NCB vs Welfarist UCB, Bernoulli arms; $p=0$}
		\end{subfigure} 
		\begin{subfigure}[b]{0.32\textwidth}
			\includegraphics[scale=0.28]{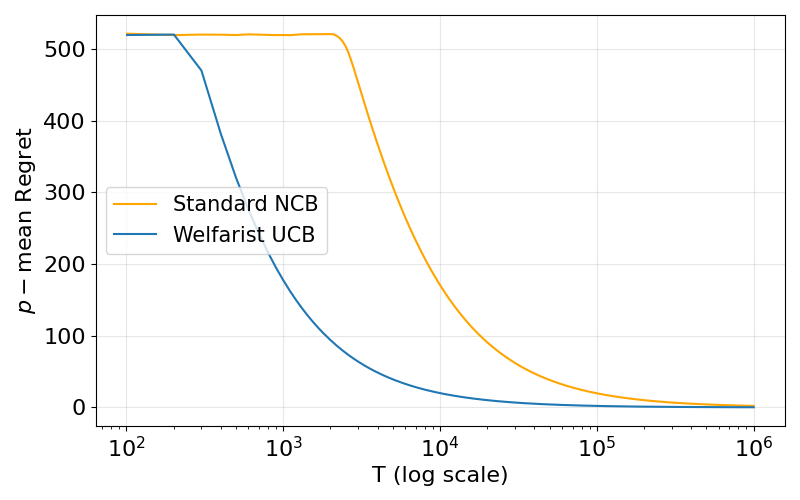}
			\caption{\centering Standard NCB vs Welfarist UCB, sub-Gaussian arms; $p=0$}
		\end{subfigure} 
		\begin{subfigure}[b]{0.32\textwidth}
			\includegraphics[scale=0.28]{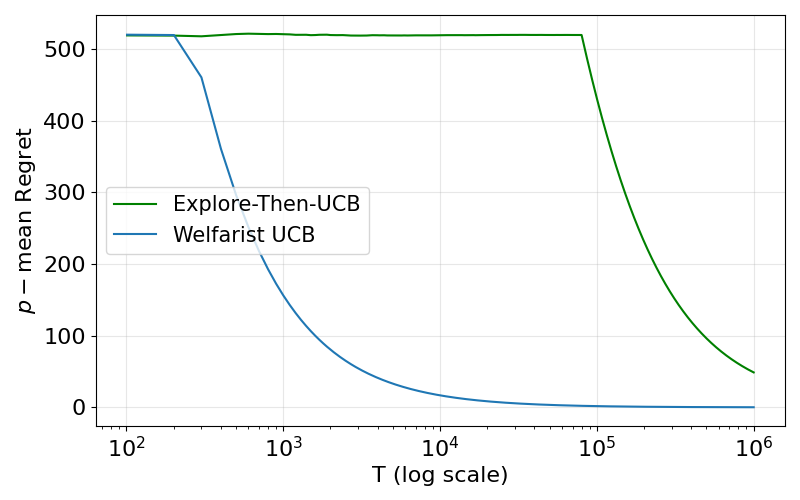}
			\caption{\centering Explore-then-UCB vs Welfarist UCB, sub-Gaussian arms; $p=0.5$}
		\end{subfigure} 
		\begin{subfigure}[b]{0.32\textwidth}
			\includegraphics[scale=0.28]{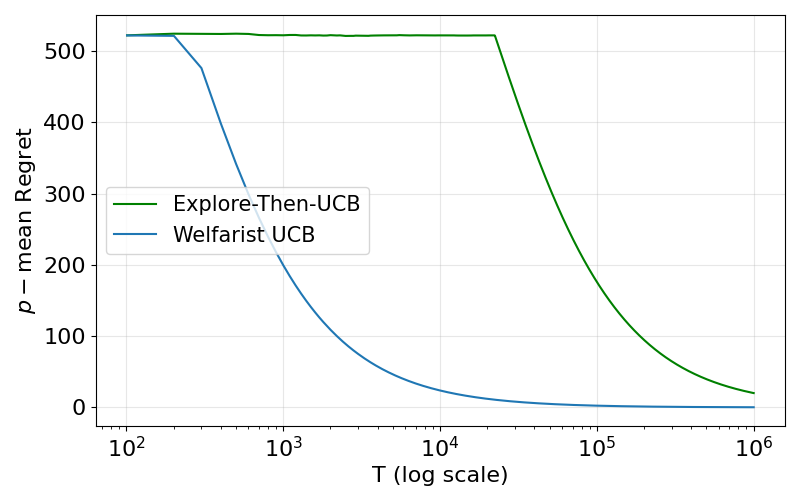}
			\caption{\centering  Explore-then-UCB vs Welfarist UCB, sub-Gaussian arms; $p=-0.5$}
		\end{subfigure} 
		\begin{subfigure}[b]{0.32\textwidth}
			\includegraphics[scale=0.28]{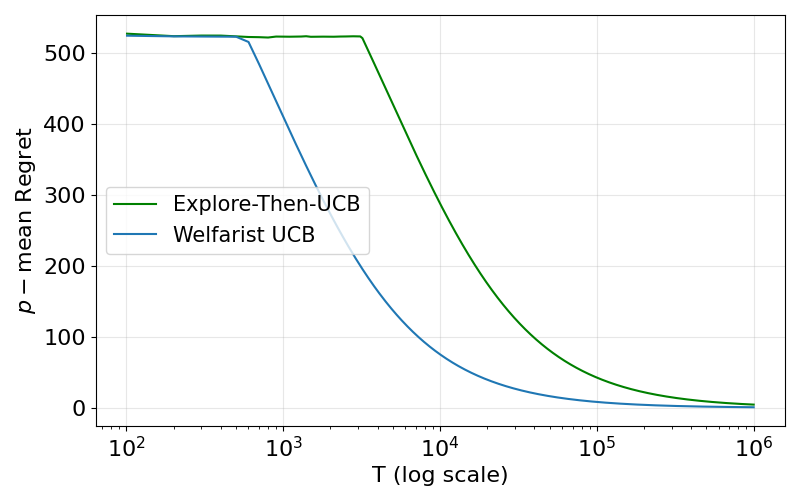}
			\caption{\centering  Explore-then-UCB vs Welfarist UCB, sub-Gaussian arms; $p=-1.5$}
		\end{subfigure} 
		\begin{subfigure}[b]{0.32\textwidth}
			\includegraphics[scale=0.28]{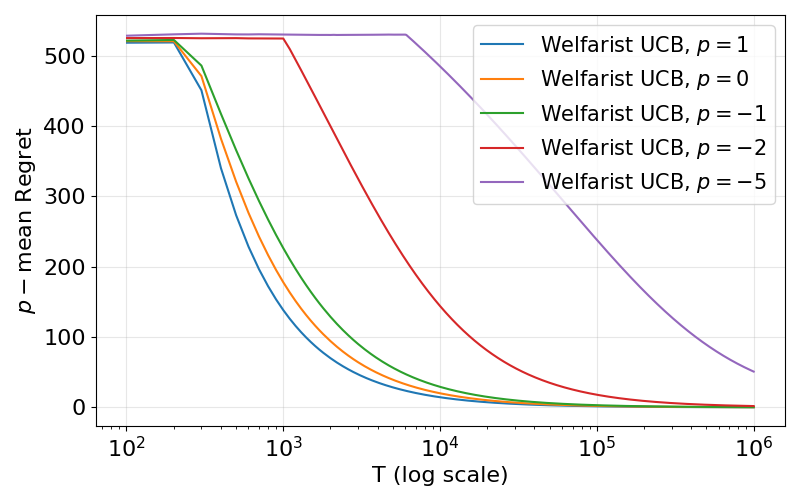}
            \caption{Welfarist UCB, \protect\\ vary $p$}

		\end{subfigure} 
    \caption{Numerical results for Welfarist UCB. (a) illustrates the comparision between the Nash regret ($p=0$) from Welfarist UCB and Standard NCB for Bernoulli rewards. (b) shows that Welfarist UCB achieves Standard NCB under sub-gaussian rewards. (c), (d) and (e) show that Welfarist UCB outperforms Explore-Then-UCB for the general $p-$mean regret in all three regimes. (f) shows the ablation on $p$; as $p$ decreases, we get more fairness, but the regret also increases.}
\end{figure*}

\section{Experiments}
\label{sec:experiments}
In this section, we present the results from our numerical simulations. 
All experiments report the average reward over 50 runs to estimate $\mathbb{E}[\mu_{I_t}]$. We compare our algorithm (Welfarist-UCB) with $\NCB$~\citep{barman2023fairness} and Explore-then-UCB~\citep{krishna2025p}. 

We first consider $k=50$ Bernoulli arms with means sampled uniformly at random from the interval $[0.005,1]$, and compare the Nash regret ($p=0$) of $\NCB$ and Welfarist-UCB. As shown in Figure~(a), the Nash regret of Welfarist-UCB decreases much faster than that of $\NCB$ as the horizon $T$ increases.

Next, we consider $k=50$ Gaussian arms with means sampled uniformly at random from the interval $[10,1000]$, and a fixed standard deviation of $\sigma=20$. As shown in Figure~(b), Welfarist-UCB minimizes the Nash regret significantly faster than $\NCB$.

Then, we compare the $p$-mean regret of Welfarist-UCB and Explore-then-UCB under the Gaussian bandit setting. We evaluate the algorithms in three regimes by selecting $p=0.5$ ($0<p \leq 1$), $p=-0.5$ ($-1<p<0$), and $p=-1.5$ ($p \leq -1$). Figures~(c), (d), and (e) present the results for these three cases, respectively, and illustrate that Welfarist-UCB consistently minimizes the $p$-mean regret faster than Explore-then-UCB across all regimes.

Finally, we conduct an ablation study on the fairness parameter $p$, evaluating the $p$-mean regret of our algorithm across different regimes. As shown in Figure~(f), the $p$-mean regret consistently increases as 
$p$ decreases, indicating that stronger fairness requirements come at the expense of higher regret. This trend corroborates the “no-free-lunch” hypothesis stated in Remark~\ref{rem:freeLunch}.

\section{Conclusion}
We show that the classic UCB algorithm, when combined with a data-adaptive exploration phase, yields a near-optimal solution for fairness-aware regret in stochastic bandits. Our Welfarist-UCB algorithm extends existing guarantees from bounded, non-negative rewards to general sub-Gaussian reward distributions, achieving near-optimal regret under both Nash and $p$-mean welfare metrics. This makes fair sequential decision-making more practical and widely applicable. Moreover, our results highlight the versatility of UCB for optimizing social welfare objectives, while also uncovering a “no-free-lunch” principle: enforcing stricter fairness criteria fundamentally increases the difficulty of the learning problem, thereby requiring more samples to achieve low regret. An important direction for future work is to formalize this principle by establishing matching lower bounds.


\appendix

\newpage
{\centering \LARGE \textbf{Appendix}}\\\hfill\\
In Appendix \ref{sec:related_work}, we discuss some of the existing literature related to multi-arm bandits and differential privacy. We
provide proofs for the results related to Nash regret in Appendix \ref{appendix:modifiedncb-supporting-lemmas}, whereas proofs for $p-$mean regret analysis are provided in \ref{app:pmeanproof}. Finally, proofs for some auxillary lemmas are provided in Appendix \ref{sec:auxLem}.

\section{Related Works}
\label{sec:related_work}
The incorporation of fairness considerations into MAB problems has garnered significant attention in recent years, driven by the increasing deployment of learning algorithms in domains with far-reaching social implications.


\paragraph{Fairness in Multi-Armed Bandits}
In the study of Multi-Armed Bandits (MABs), fairness has been conceptualized in several distinct ways. One prominent approach, explored by \citet{joseph2016fairness}, \citet{celis2019controlling}, and \citet{patil2021achieving}, centers on guaranteeing equitable treatment for the arms themselves. Another direction, particularly in multi-agent contexts, has focused on the fair allocation of rewards among different agents from each arm pull, as investigated by \citet{hossain2021fair} and \citet{jones2023efficient}. Our research diverges from these perspectives by addressing fairness across time, where each sequential round is uniquely treated as a distinct agent requiring fair consideration.


A foundational work by \citet{barman2023fairness} established the notion of Nash regret and proposed the Nash Confidence Bound algorithm to minimize it in stochastic multi-armed bandit environments. Their algorithm achieves tight regret guarantees that hold for both known and unknown ($T$-oblivious) time horizons. However, this approach requires a specialized algorithm and relies on strong assumptions that limit its applicability. Specifically, its analysis is based on multiplicative concentration inequalities, which restricts rewards to be \textbf{bounded and non-negative}, making the method unsuitable for general distributions like Gaussian rewards. Additionally, their algorithm implicitly requires prior knowledge of an upper bound on the optimal mean reward, $\mu^{*}$.

\citet{krishna2025p} later extended this by generalizing the objective to $p$-mean regret, which is derived from the $p$-mean welfare function in social choice theory. They proposed using a standard Explore-then-UCB algorithm. However, to circumvent the known failure of UCB for Nash regret shown by \citet{barman2023fairness}, their analysis relies on a restrictive assumption that all arms possess a minimum expected reward of at least $\tilde{\Omega}(\sqrt{k}T^{-1/4})$. This assumption significantly limits the applicability of their method, as it excludes common bandit instances. The justification for this assumption has also been noted as unsatisfactory, as it conflates different rates of convergence. Furthermore, their approach is confined to bounded, non-negative rewards and achieves sub-optimal regret bounds in certain fairness regimes, such as $\tilde{O}(k^{3/4}T^{-1/4})$ for $p \in [-1, 0)$

\paragraph{p-Mean Welfare and Fair Division}
The concept of $p$-mean welfare is well-established within the field of fair division, which integrates principles from both mathematical economics and computer science. Drawing from social choice theory~\cite{moulin2004fair}, this welfare function serves as a tunable framework for navigating the trade-off between equity and efficiency, a topic explored in various works~\cite{barman2020tight, garg2021tractable, barman2022universal,eckart2024fairness}. It is axiomatically defined by five fundamental properties—anonymity, scale invariance, continuity, monotonicity, and symmetry—which together ensure its alignment with core principles of fair allocation. Furthermore, it adheres to the Pigou-Dalton principle, which states that welfare increases when a resource is transferred from a more advantaged individual to a less advantaged one; this principle constrains the parameter $p$ to values no greater than 1. By building on this robust theoretical grounding, our work avoids the need to formulate arbitrary or ad-hoc fairness rules.

\paragraph{Other Related Work}
The pursuit of fairness is expanding into related domains, underscoring a broader trend of integrating such considerations into machine learning algorithms. For instance, \citet{sawarni2024nash} investigated Nash regret within the context of stochastic linear bandits, for which they established tight upper bounds under the condition of sub-Poisson rewards. 

The research by \citet{zhang2024no} on online Nash social welfare (NSW) maximization provides another relevant perspective. While their approach differs from ours—focusing on decisions that simultaneously affect multiple agents rather than our sequential, round-by-round fairness model—it similarly emphasizes the critical need to embed fairness into online decision-making frameworks.

Further illustrating this trend, \citet{mandal2022socially} apply an axiomatic framework to multi-agent reinforcement learning, showing that Nash Social Welfare uniquely satisfies specific fairness criteria and deriving regret bounds for policies that optimize for fair outcomes. Collectively, such studies demonstrate the increasing integration of fairness metrics like NSW into a wide array of learning algorithms, spanning from bandit problems to more complex Markov decision process environments.

Recent work has begun to address the intersection of privacy and fairness in sequential decision-making. \cite{sarkar2025dp} bridge the gap between privacy-preserving and fairness-aware bandits by introducing the Differentially Private Nash Confidence Bound (DP-NCB) framework. Their work provides a unified approach to simultaneously guarantee $\epsilon$-differential privacy while minimizing Nash regret. The proposed algorithms operate under both global (GDP) and local (LDP) privacy models and are designed to be anytime, requiring no knowledge of the time horizon. For the GDP setting, they achieve an order-optimal Nash regret of $\tilde{O}(\sqrt{\frac{k}{T}} + \frac{k}{\epsilon T})$, and in the more restrictive LDP setting, they achieve a regret of $\tilde{O}(\sqrt{\frac{k}{T}} + \frac{1}{\epsilon}\sqrt{\frac{k}{T}})$, with both bounds matching known lower bounds up to logarithmic factors.
\section{Missing Proofs}

\subsection{Proof for Bound on Probability of Good Event $\cE$}
\label{appendix:proof-of-E}
To prove this result, we will instead bound $\cE^c$. Invoking Lemma \ref{lem:hoeffding} with $\epsilon = 2\sqrt{\frac{2\sigma^2\log T}{n_i}}$, we have, for every arm $i$,
\begin{align*}
    \prob \left\{ |\hat{\mu_i}-\mu_i| \geq   2\sqrt{\frac{2\sigma^2\log T}{n_i}} \right\} & \leq 2\exp \left( -\frac{8 n_i \sigma^2\log T}{2n_i\sigma^2} \right)=\frac{2}{T^{4}}
\end{align*}
Thus, by union bound, we get
\begin{align}
    \prob \left\{ \cE_2^c \right\} = \frac{2}{T^{4}}\cdot kT \leq \frac{2}{T} \label {ineq:E_2c}
\end{align}
Finally, we have
\begin{align*}
    \prob \{\cE\} = 1-\prob\{\cE^c\} \geq 1-\frac{2}{T}~.
\end{align*}

\subsection{Proof of Supporting Lemmas}
\label{appendix:modifiedncb-supporting-lemmas}

First, we state the following two lemmas, which help us derive the terminating condition and prove Lemma \ref{tau:anytime}.

\begin{restatable}{lemma}{LemmaTauLower}
\label{tau:lower}
 Under the event $\cE$, for any arm $i$ and its sample count $n_i \leq 32  \Sample$, we have $n_i \leq 128p_a^2\sigma^2\frac{\log T}{\left(\widehat{\mu}_{i, n} - 2\sqrt{\frac{2\sigma^2\log T}{n_{i}}}\right)^2}$. 
\end{restatable}
\begin{proof}
Write $N \coloneqq 32 \Sample$. Note that, for any arm $i$, the quantity $n \  \widehat{\mu}_{i,n} $ is equal to the sum of the rewards observed for arm $i$ in the first $n$ samples. Therefore, for all $n \leq N$, we have 
\begin{align}
n \  \widehat{\mu}_{i,n} &\leq n \ \big(\mu_i + 2\sqrt{\frac{2\sigma^2\log T}{n}}\big) \tag{via event $\cE$}\\
    &= n\mu_i + 2\sqrt{2n\sigma^2\log T} \leq 128p_a^2 \sigma^2\frac{\mu_i\log T}{(\mu^*)^2} + 2\sqrt{2n\sigma^2\log T} \tag{$n \leq 32S$}\\ 
    & \leq 128p_a^2\sigma^2 \frac{\log T}{\mu^*} + 2\sqrt{2n\sigma^2\log T}\tag{$\mu_i\leq\mu^*$} \\ 
    &\leq 128p_a^2\sigma^2\frac{\log T}{\widehat{\mu}_{i, n} - 2\sqrt{\frac{2\sigma^2\log T}{n}}} + 2\sqrt{2n\sigma^2\log T}~,
\nonumber\\
\implies& n\left(\widehat{\mu}_{i,n}-2\sqrt{\frac{2\sigma^2\log T}{n}}\right)\leq 128p_a^2\sigma^2\frac{\log T}{\big(\widehat{\mu}_{i,n}-2\sqrt{\frac{2\sigma^2\log T}{n}}\big)}\nonumber\\
\implies n&\leq 128p_a^2\sigma^2\frac{\log T}{\big(\widehat{\mu}_{i,n}-2\sqrt{\frac{2\sigma^2\log T}{n}}\big)^2}
\end{align}
which completes the proof.
\end{proof}

\begin{restatable}{lemma}{LemmaTauUpper}
\label{tau:upper}
Under the event $\cE$, for the optimal arm $i^*$ and its sample count $n_{i^*} \geq 128  \Sample$, we have $n_{i^*} \geq 256p_a^2\sigma^2\frac{\log T}{\left(\widehat{\mu}_{i^*, n} - 2\sqrt{\frac{2\sigma^2\log T}{n_{i^*}}}\right)^2}$.
\end{restatable}
\begin{proof}
Write $M \coloneqq 128 \Sample$ and note that, for all $n \geq M$, we have $n \ \widehat{\mu}_{i,n} \geq M \ \widehat{\mu}_{i,M}$. Thus, proving this lemma for $M=128S$ is sufficient. Observe that 
\begin{align*}
\widehat{\mu}_{i^*,M} & \geq \mu^*-2\sqrt{\frac{2\sigma^2\log T}{M}} \tag{via event $\cE$} \\
& =\mu^*-\frac{\mu^*}{\sqrt{64p_a^2}}      \tag{since $M=128 \Sample = \frac{512p_a^2\sigma^2 \log T}{(\mu^*)^2}$}\\
& = \mu^*\left(1-\frac{1}{8|p_a|}\right) \geq \frac{7}{8}\mu^* \tag{$|p_a| \geq 1$}
\end{align*}
Thus, the total observed reward satisfies:
\begin{align}
    M \ \widehat{\mu}^*_{i,M} \geq \frac{7}{8}  \mu^* \ 128 \Sample=\frac{448 p_a^2 \sigma^2 \log T}{\mu^*} \label{ineq:rewardIneq}
\end{align}
Now, consider the following terms:
\begin{align}
256 p_a^2 \sigma^2\frac{\log T}{\widehat{\mu}_{i^*, M} - 2\sqrt{\frac{2\sigma^2\log T}{M}}} \leq 256p_a^2\sigma^2 \frac{\log T}{\frac{6}{8} \mu^*} =384p_a^2\sigma^2 \frac{\log T}{ \mu^*}\label{ineq:lemFirst}
\end{align}
and,
\begin{align}
    2\sqrt{2\sigma^2M\log T} = \frac{64\sigma^2\log T}{\mu^*} \leq  \frac{64p_a^2\sigma^2\log T}{\mu^*} \tag{$|p_a| \geq 1$}\label{eq:lemSecond}
\end{align}
Adding inequality \eqref{ineq:lemFirst} and the above inequality, we have
\begin{align}
   & 256 p_a^2 \frac{\sigma^2\log T}{\widehat{\mu}_{i^*, M} - 2\sqrt{\frac{2\sigma^2\log T}{M}}} + 2\sqrt{2\sigma^2M\log T} \leq \frac{448p_a^2\sigma^2\log T}{\mu^*} \leq M \ \widehat{\mu}^*_{i,M} \tag{via inequality \eqref{ineq:rewardIneq}}\\
    &\qquad\qquad\qquad\implies M \geq 256p_a^2\sigma^2\frac{\log T}{\left(\widehat{\mu}_{i^*, M} - 2\sqrt{\frac{2\sigma^2\log T}{M}}\right)^2}
\end{align}
This completes the proof of the lemma.
\end{proof}
Next, we restate and prove Lemma \ref{tau:anytime}.

\LemmaTauAnytime*
\begin{proof}
Suppose Phase I terminates after $t = 32kS$ rounds. Then each arm has been pulled $32S$ times. Thus, via Lemma \ref{tau:lower}, for every arm $i$, the number of pulls $n_i$ is less than $128p_a^2\sigma^2\frac{\log T}{\left(\widehat{\mu}_{i, n} - 2\sqrt{\frac{2\sigma^2\log T}{n_{i}}}\right)^2} \leq 192p_a^2\sigma^2\frac{\log T}{\left(\widehat{\mu}_{i, n} - 2\sqrt{\frac{2\sigma^2\log T}{n_{i}}}\right)^2}$. Therefore, $\tau \geq 32 k  \Sample$.

Similarly, suppose Phase I terminates after $t = 128kS$ rounds. Then each arm has been pulled $128S$ times. Therefore, Lemma \ref{tau:upper} implies that, by round $t_2$ and for the optimal arm $i^*$, the number of pulls $n_{i^*}$  is at least $ 256p_a^2\sigma^2\frac{\log T}{\left(\widehat{\mu}_{i^*, n} - 2\sqrt{\frac{2\sigma^2\log T}{n_{i^*}}}\right)^2}\geq 192p_a^2\sigma^2\frac{\log T}{\left(\widehat{\mu}_{i^*, n} - 2\sqrt{\frac{2\sigma^2\log T}{n_{i^*}}}\right)^2}$. Hence, $\tau \leq 128 kS$
\end{proof}


\begin{lemma}[Permutation sampling is a static coupling of sequential uniform sampling]
\label{lem:permute}
Let $\mathcal B$ be a multiset containing $N_i$ items of type $i$ for $i=1,\dots,K$, with $\sum_{i=1}^K N_i=N$.  
Consider two sampling schemes:
\begin{enumerate}
  \item[(A)] (Permutation): Draw a uniform random permutation of the $N$ (labelled) items and set $I_t$ to be the type observed at position $t$ (sampling \emph{without} replacement).
  \item[(B)] (Sequential with replacement): Draw each $I_t$ independently \emph{with} replacement, choosing type $i$ with probability $N_i/N$ at every draw.
\end{enumerate}
Then for every $t\in\{1,\dots,N\}$ and every type $i$, the one-step marginals coincide, i.e., 
\[
\Pr_A(I_t=i)=\Pr_B(I_t=i)=\frac{N_i}{N}~.
\]
\end{lemma}

\begin{proof}
Scheme (B) is immediate by definition: each draw chooses type $i$ with probability $N_i/N$, so $\Pr_B(I_t=i)=N_i/N$ for every $t$.

For scheme (A) consider the uniform permutation of the $N$ labelled items (items of the same type can be thought of as distinct labels that share the same type). By symmetry, every labelled item is equally likely to occupy position $t$ in the permutation. Since there are $N_i$ labelled items of type $i$ out of $N$ total labelled items, the probability that position $t$ holds an item of type $i$ is
\[
\Pr_A(I_t=i)=\frac{N_i}{N}.
\]
Thus $\Pr_A(I_t=i)=\Pr_B(I_t=i)=N_i/N$, as claimed.
\end{proof}


\begin{restatable}{lemma}{LemmaBoundNSW}
\label{lem:modified_ncb}
Consider a bandit instance with optimal mean $\mu^* \geq  \frac{40\sigma \sqrt{2k\log T\log k}}{\sqrt T}$. Then, we have 
\begin{align*}
\left(\prod_{t=1}^{T} \E \left[ \mu_{I_t} \right] \right)^\frac{1}{T}  \geq \mu^*-256\sigma\sqrt{\frac{ k \log T \log k}{T} }-\frac{4\mu^*k\log T\log k}{T}-\frac{8\mu^*}{T}.
\end{align*}
\end{restatable}

\begin{proof}
Firstly, for {Nash regret}, the parameter $p$ is set to $0$. Consequently, we will take $p_a=1$ wherever necessary in the analysis. Next, we assume that {Phase 1} runs for at most $\overline{T} \leq 128 k \mathcal{S}$ rounds. The existence of this upper bound on $\overline{T}$ is guaranteed by {Lemma \ref{tau:anytime}}, which ensures that Algorithm \ref{algo:ucb:modified} will complete Phase 1 by the $\overline{T}$-th round; specifically, the termination condition of the first while-loop (Line \ref{step:PhaseOneAlgTwo}) in the algorithm must be satisfied by the $\overline{T}$-th round. Also, note that, the following inequality holds based on the assumption on $\mu^*$ stated in the lemma.
\begin{align}
\frac{\overline{T} \ \log \left(k \right)}{T} =  \frac{128 \ k \Sample \ \log \left( k\right) }{T} 
= \frac{512 \ k \sigma^2 \log T \ \log \left( k\right) }{(\mu^*)^2 T} \leq \frac{512}{3200}\leq 1 \label{ineq:overlineTexp}
\end{align} 
Next, we split the Nash social welfare into the following two terms, for the two phases respectively: 
\begin{align}
    \left(\prod_{t=1}^{w} \E \left[ \mu_{I_t} \right] \right)^\frac{1}{T} =  \left(\prod_{t=1}^{\overline{T}} \E \left[ \mu_{I_t} \right]\right)^\frac{1}{T} \left(\prod_{t=\overline{T} + 1}^{w}  \E \left[ \mu_{I_t} \right] \right)^\frac{1}{T} \label{eqn:splitW}
\end{align} 
We shall lower-bound these two products separately.

The term corresponding to Phase I, i.e. first term in the RHS of equation (\ref{eqn:splitW}) can be bounded as:
\begingroup
\allowdisplaybreaks
\begin{align}
    \left(\prod_{t=1}^{\overline{T}} \E [\mu_{I_t} ] \right)^\frac{1}{T} 
     & \geq \left(1-\frac{2}{T}\right)^{\frac{\overline{T}}{T}}\left(\frac{\mu^*}{k}\right)^{\frac{\overline{T}}{T}}      \nonumber       \geq \left(1-\frac{2}{T}\right) \left(\mu^*\right)^{\frac{\overline{T}}{T}} \left(\frac{1}{k}\right)^{\frac{\overline{T}}{T}} \nonumber                \\
     & = \left(1-\frac{2}{T}\right)\left(\mu^*\right)^{\frac{\overline{T}}{T}} \left(\frac{1}{2}\right)^{\frac{\overline{T}  \log (k)}{T}} \nonumber       = \left(1-\frac{2}{T}\right)\left(\mu^*\right)^{\frac{\overline{T}}{T}} \left(1- \frac{1}{2}\right)^{\frac{\overline{T}  \log (k)}{T}} \nonumber   \\
     & \geq  \left(\mu^*\right)^{\frac{\overline{T}}{T}} \left(1-{\frac{\overline{T}  \log (k)}{T}}\right)\left(1-\frac{2}{T}\right)    \label{ineq:phaseone:anytime}
\end{align}
\endgroup
To prove the last inequality, observe that the exponent ${\frac{\overline{T}  \log (k)}{T}} \leq 1$ (see inequality (\ref{ineq:overlineTexp})). Thus, we can apply Claim \ref{lem:binomial}.

Now, for the Phase II term, i.e. the second term in the RHS of equation (\ref{eqn:splitW}), we have the following
\begin{align}
    \left(\prod_{t=\overline{T}+ 1}^{T} \E\left[ \mu_{I_t} \right]\right)^\frac{1}{T}
     & \geq \E\left[\left( \prod_{t=\overline{T}+1}^{T} \mu_{I_t} \right)^\frac{1}{T}\right ] \tag{Multivariate Jensen's inequality} \nonumber               \\
     & \geq \E \left[\left( \prod_{t=\overline{T}+1}^{T} \mu_{I_t} \right)^\frac{1}{T} \;\middle|\; \mathcal{E} \right]  \prob\{ \mathcal{E} \} \label{ineq:interim:anytime}
\end{align}
Since, as mentioned, {Lemma \ref{tau:anytime}} guarantees that {Algorithm \ref{algo:ucb:modified}} completes {Phase 1} by the $\overline{T}$-th round, any round $t > \overline{T}$ necessarily falls under {Phase 2}.

To bound the expected value on the right-hand side of {Inequality (\ref{ineq:interim:anytime})}, we consider only the arms that are pulled at least once after the first $\overline{T}$ rounds (i.e., in Phase 2). With re-indexing, let $\{1,2, \ldots, \ell\}$ denote this set of arms. Let $m_i \geq 1$ be the number of times arm $i \in [\ell]$ is pulled in Phase 2, such that $\sum_{i=1}^\ell m_i = T - \overline{T}$.

Furthermore, let $T_i$ denote the total number of times arm $i$ is pulled throughout the algorithm. Note that $(T_i - m_i)$ is the number of times arm $i \in [\ell]$ is pulled during the first $\overline{T}$ rounds (Phase 1).

Using this notation, the expected value can be rewritten as:
$$
\E \left[\left( \prod_{t=\overline{T}+1}^{T} \ \mu_{I_t} \right)^\frac{1}{T} \;\middle|\; E \right] = \E\left[\left( \prod_{i=1}^{\ell} \mu_{i}^\frac{m_i}{T} \right)\;\middle|\; E \right]
$$
Moreover, since our analysis is conditioned on the good event $E$, we can apply {Lemma \ref{lem:suboptimal_arms:anytime}} to each arm $i \in [\ell]$. Hence,
\begin{align}
    \E \left[\left( \prod_{t=\overline{T}+1}^{T} \ \mu_{I_t} \right)^\frac{1}{T} \;\middle|\; \cE \right]  = \E\left[\left( \prod_{i=1}^{\ell} \mu_{i}^\frac{m_i}{T}  \right)\;\middle|\; \cE \right] \nonumber     
    & \geq \E\left[\prod_{i=1}^{\ell}\left(\mu^* - 2\sqrt{\frac{ 2\sigma^2\log T}{T_i -1 }} \right)^\frac{m_i}{T} \;\middle|\; \cE \right]   \tag{Lemma \ref{lem:suboptimal_arms:anytime}} \nonumber\\
    &= (\mu^*)^{\frac{T-\overline{T}}{T}} \E\left[\prod_{i=1}^{\ell}\left(1 - 2\sqrt{\frac{ 2\sigma^2\log T}{(\mu^*)^2(T_i -1)}} \right)^\frac{m_i}{T}	\;\middle|\; \cE \right] \label{ineq:dunzo:anytime}
\end{align}
For the last equality, we use $\sum_{i=1}^\ell m_i = T - \overline{T}$. 
We now state a numeric inequality, which helps us simplify the analysis. The proof is deferred to the Auxillary Lemmas section.
\begin{restatable}{claim}{LemmaNumeric}
\label{lem:binomial}
For all reals $x \in \left[0, \frac{1}{2}\right]$ and all $a \geq0$, we have $(1-x)^{a} \geq  1- 2ax$.
\end{restatable}
Recall that each arm is pulled at least $32 S$ times during the first $\overline{T}$ rounds. Hence, $T_i > 32 S$, for each arm $i \in [\ell]$.
Since $S = \frac{4\sigma^2 \log T}{(\mu^*)^2}$, we have:
\begin{align*}
2\sqrt{\frac{2\sigma^2 \log T}{(\mu^*)^2(T_i - 1)}} \leq 2 \sqrt{\frac{2\sigma^2\log T}{(\mu^*)^2(32S)}} = 2 \sqrt{\frac{2\sigma^2\log T}{(\mu^*)^2 \cdot 32 \cdot \frac{4\sigma^2 \log T}{(\mu^*)^2}}} = 2 \sqrt{\frac{2\sigma^2\log T}{128\sigma^2\log T}} \leq \frac{1}{2} \quad \text{for each } i \in [\ell].
\end{align*}
Therefore, we can apply Claim \ref{lem:binomial} to reduce the expected value in inequality (\ref{ineq:dunzo:anytime}) as follows
\begin{align*}
    &\E \left[\prod_{i=1}^{\ell}\left(1 - \frac{2}{\mu^*}\sqrt{\frac{ 2\sigma^2\log T}{(T_i -1 )}} \right)^\frac{m_i}{T}\;\middle|\; \cE \right] \geq \E \left[\prod_{i=1}^{\ell}\left(1 - \frac{4 \ m_i}{T}\frac{1}{\mu^*}\sqrt{\frac{2\sigma^2 \log T}{(T_i -1)}} \right)\;\middle|\; \cE \right]\\
    & \geq \E \left[\prod_{i=1}^{\ell}\left(1 - \frac{4}{T\mu^*}\sqrt{ 2m_i\sigma^2 \log T} \right)\;\middle|\; \cE \right]     & \tag{since $T_i \geq m_i + 1$}
\end{align*}
We can further simplify the above inequality by noting that $(1-x)(1-y) \geq 1 - x- y$ for all $x,y \geq 0$.
\begingroup
\allowdisplaybreaks
\begin{align*}
    &\E\left[\prod_{i=1}^{\ell}\left(1 - \frac{4}{T\mu^*}\sqrt{2m_i \sigma^2\log T } \right)\;\middle|\; \cE \right]  \geq \E\left[1 -\sum_{i=1}^{\ell}\left(\frac{4}{T\mu^*}\sqrt{2m_i\sigma^2 \log T } \right)\;\middle|\; \cE \right]\\
    &= 1 -\left(\frac{4}{T\mu^*}\sqrt{ 2\sigma^2\log T} \right) \E\left[ \sum_{i=1}^{\ell} \sqrt{m_i}\;\middle|\; \cE \right]          \geq 1 -\left(\frac{4}{T\mu^*}\sqrt{2\sigma^2\log T } \right) \E\left[ \sqrt{\ell} \ \sqrt{\sum_{i=1}^\ell m_i} \;\middle|\; \cE \right] \tag{Cauchy-Schwarz inequality}\\
    &  \geq 1 -\left(\frac{4 }{T\mu^*}\sqrt{ 2\sigma^2\log T} \right) \E\left[ \sqrt{\ell \ T} \;\middle|\; \cE \right] \tag{since $\sum_i m_i \leq T$}                         = 1 -\left( \frac{4}{\mu^*} \sqrt{\frac{2\sigma^2 \log T }{ T }} \right) \E\left[ \sqrt{\ell} \;\middle|\; \cE \right] \\
    &\geq 1 - \left( \frac{4}{\mu^*}\sqrt{\frac{ 2k \sigma^2\log T }{ T }} \right) \tag{since $\ell \leq k$}
\end{align*}
\endgroup
Using this bound, along with inequalities (\ref{ineq:interim:anytime}), and (\ref{ineq:dunzo:anytime}), we obtain 
\begin{align}
    \left(\prod_{t=\overline{T} + 1}^{T} \E\left[ \mu_{I_t} \right]\right)^\frac{1}{T} \geq (\mu^*)^{\frac{T-\overline{T}}{T}} \left ( 1 -\frac{4}{\mu^*}\sqrt{\frac{ 2k \sigma^2\log T }{T}}  \right) \prob \{\cE\} \label{ineq:toomany:anytime}
\end{align}
Thus, inequalities (\ref{ineq:toomany:anytime}) and (\ref{ineq:phaseone:anytime}) to obtain relevant bounds on the NSW from Phase I and II respectively.
Hence, for the combined Nash social welfare of the algorithm, we have
\begin{align*}
    \left(\prod_{t=1}^{T} \E \left[ \mu_{I_t} \right] \right)^\frac{1}{T}
     & \geq \mu^*\left(1-{\frac{\overline{T}\cdot \log (k)}{T}}-\frac{2}{T}\right) \left( 1 -\frac{4}{\mu^*}\sqrt{\frac{2 k \sigma^2\log T }{T} } \right) \prob\{ \cE \}                   \\
     & \geq \mu^*\left(1-{\frac{\overline{T}\cdot \log (k)}{T}}-\frac{2}{T}\right) \left( 1 -\frac{4}{\mu^*}\sqrt{\frac{ 2k \sigma^2\log T }{T} } \right) \left( 1 - \frac{2}{T} \right) \tag{via good event $\cE$}
\end{align*}
Using the inequality $(1-x)(1-y)\geq 1-x-y \ \forall\ x,y\geq0$ we have,
\begin{align*}
      \left(\prod_{t=1}^{T} \E \left[ \mu_{I_t} \right] \right)^\frac{1}{T}
     & \geq \mu^* \left(1-{\frac{\overline{T}\cdot \log (k)}{T}}-\frac{4}{\mu^*}\sqrt{\frac{2 k \sigma^2\log T }{T} } -\frac{4}{T}\right)   \\  
     & \geq \mu^* \left(1-\frac{\log \left( k\right)}{T}\frac{512 \ k\sigma^2 \log T \  }{(\mu^*)^2}-\frac{4}{\mu^*}\sqrt{\frac{2 k \sigma^2 \log T }{T} } -\frac{4}{T}\right)   \\ 
     & \geq \mu^* \left(1-\frac{512 \ k \sigma^2 \log T \log \left( k\right)\  }{(\mu^*)^2T}-\frac{4}{\mu^*}\sqrt{\frac{2 k \sigma^2 \log T }{T} } -\frac{4}{T}\right)   \\ 
     &= \mu^* \left(1-\frac{256\ \sqrt{k\sigma^2 \log T\log k }}{\mu^* \sqrt{T}} \cdot \frac{2 { \sqrt{k \sigma^2 \log T\log k  }}}{\mu^* \sqrt{T}}-\frac{4k\log T\log k}{T}-\frac{4}{\mu^*}\sqrt{\frac{2 k \sigma^2\log T }{T} } -\frac{4}{T}\right)\\
     &\geq \mu^* \left(1-\frac{256 \sqrt{k\sigma^2 \log T\log k }}{\mu^* \sqrt{T}} -\frac{4}{\mu^*}\sqrt{\frac{2 k \sigma^2\log T }{T} } -\frac{4}{T}\right)\tag{} 
\end{align*}
Note that the last inequality holds since $\frac{2 {\sqrt{k \sigma^2\log T \log k }}}{\mu^* \sqrt{T}} \leq 1$ for large enought $T$ as $\mu^* \geq \frac{ { 40\sigma\sqrt{2k \log T \log k }}}{\sqrt{T}}$. Thus, we get
\begin{align*}
    \left(\prod_{t=1}^{T} \E \left[ \mu_{I_t} \right] \right)^\frac{1}{T}& \geq \mu^*-256 \sigma  \sqrt{\frac{ k \log T \log k }{T} }-\frac{4\mu^*k\log T\log k}{T}-\frac{4\mu^*}{T}.
\end{align*}
The lemma stands proved. 
\end{proof}

\subsection{Proofs for $p$-means}
\label{app:pmeanproof}
We first state a fundamental inequality, known as the \emph{power inequality} or the \emph{generalised mean inequality}, in the following lemma.

\begin{lemma}\label{lem:genMeanIneq}
Let $x_1,\dots,x_n \ge 0$ and for $r\in\mathbb{R}$ define
\[
M_r =
\begin{cases}
\left(\tfrac{1}{n}\sum_{i=1}^n x_i^r\right)^{1/r}, & r\neq 0, \\[6pt]
\left(\prod_{i=1}^n x_i\right)^{1/n}, & r=0.
\end{cases}
\]
Then $M_r$ is strictly increasing in $r$; in particular, if $a<b$ then $M_a<M_b$.
\end{lemma}

\TheoremPLTZero*

\begin{proof}
We split our analysis into three cases, where $p> 0$, and $p \leq 0$ respectively
\subsubsection*{When $p \geq 0$}
Invoking the generalised mean inequality (Lemma \ref{lem:genMeanIneq} for $a=0$ and $b=p>0$, we have
\begin{align*}
    \left(\prod_{t=1}^{T} \E \left[ \mu_{I_t} \right] \right)^\frac{1}{T}\leq\left(\frac{\sum_{t=1}^T(\E_{I_t}[\mu_{I_t}])^p}{T}\right)^\frac{1}{p}
\end{align*}
Thus, we get the following bound on the p-mean regret using Lemma \ref{lem:modified_ncb}
\begin{align*}
    \mathrm{R}^{p}_T & \triangleq \mu^* - \left(\frac{\sum_{t=1}^T(\E_{I_t}[\mu_{I_t}])^p}{T}\right)^\frac{1}{p} \leq \mu^* - \left(\prod_{t=1}^{T} \E \left[ \mu_{I_t} \right] \right)^\frac{1}{T} \\
    & \leq \mu^* - \left(\mu^*-256\sigma \sqrt{\frac{ k \log T\log k  }{T} }-\frac{4\mu^*k\log T\log k}{T}-\frac{8\mu^*}{T}\right)\\
    & \leq O\left(\sigma\sqrt{\frac{k\log T\log k}{T}}\right)
\end{align*}

\subsubsection*{When $p < 0$}
Notice that when $\mu^* \leq \frac{40\sigma k^{\frac{|p|+1}{2}}\sqrt{\log T}}{\sqrt T}$, the p-mean regret satisfies
\begin{align}
    \mathrm{R}^{q}_T = \mu^* - \left(\frac{\sum_{t=1}^T(\E_{I_t}[\mu_{I_t}])^p}{T}\right)^\frac{1}{p} \leq \mu^* \leq O\left(\frac{\sigma |p|k^\frac{|p|+1}{2}\sqrt{\log T}}{\sqrt T}\right)\label{ineq:muLess}
\end{align}
Next, we consider the case when $\mu^* \geq \frac{40|p|\sigma k^{\frac{|p+1|}{2}}\sqrt{\log T}}{\sqrt T}$. Firstly, we set $q=-p$ for notational convenience. 
Hence, we have $|p_a|=|p|=q$ wherever required. Thus, our objective is converted to minimising the following quantity
\begin{align}
    \mathrm{R}^{q}_T \triangleq \mu^* - \left(\frac{T}{\sum_{t=1}^T \frac{1}{\left(\E_{I_t}[\mu_{I_t}]\right)^q}}\right)^\frac{1}{q}
\end{align}
We will refer to the above quantity as $q-$regret. Towards analyzing $\mathrm{R}^{q}_T$, we first set the threshold for Phase I runs as $\overline{T}=128kS$ (i.e., the maximum length of Phase I). Then define
\begin{align*}
x \triangleq \frac{T}{\sum_{t=1}^{\overline{T}} \frac{1}{\E_{I_t}[\mu_{I_t}]^q}} &\mbox{ and } y \triangleq \frac{T}{\sum_{t=\overline{T}+1}^{T} \frac{1}{\E_{I_t}[\mu_{I_t}]^q}},
\end{align*}
so that we have
\begin{equation}\label{eqn:npreg_decomp}
	\mathrm{R}^{q}_T = \mu^* - \left(\frac{1}{\frac{1}{x} + \frac{1}{y}}\right)^{1/q}.
\end{equation}

To obtain an upper bound for $\mathrm{R}^{q}_T$, we need to upper bound $\frac{1}{x}$ and $\frac{1}{y}$. Let us start by focusing on $\frac{1}{x}$. By uniform exploration in Phase I, we have 
\begin{align*}
\E_{I_t}[\mu_{I_t}] \geq \frac{\mu^*}{k} \Leftrightarrow \frac{1}{\left(\E_{I_t}[\mu_{I_t}]\right)^q} \leq \left(\frac{k}{\mu^*}\right)^q.
\end{align*}
Hence,
\begin{equation}\label{eqn:npreg_x_inv_bound}
	\frac{1}{x} \leq \frac{\overline{T}k^q}{(\mu^*)^q T}.
\end{equation}
Next, we will focus on $\frac{1}{y}$. Note that when we condition the expectation on the good event $\cE$, the following inequality holds trivially
\begin{align*}
\E_{I_t}[\mu_{I_t}] &\geq \Pr\{\mathcal{E}\}\E_{I_t}[\mu_{I_t}|\mathcal{E}] 
\end{align*}
Hence,
\begin{align}\label{eqn:npreg_y_bound}
	y & \geq \frac{T}{\sum_{t=\overline{T}+1}^{T} \frac{1}{(\E_{I_t}[\mu_{I_t}|\mathcal{E}] \Pr\{\mathcal{E}\})^q}} =\frac{T(\Pr\{\mathcal{E}\})^q}{\sum_{t=\overline{T}+1}^{T} \frac{1}{(\E_{I_t}[\mu_{I_t}|\mathcal{E}] )^q}}\\
\end{align}

Now, we know that by Jensen's inequality, $f(z) = z^{-\frac{1}{q}}$ is convex on $\mathbb{R}_{>0}$, for $q > 0$. Utilising this result and the linearity of expectation, we get
\begin{align*}
\frac{1}{y} \leq  \frac{\sum_{t=\overline{T}+1}^{T} \frac{1}{(\E_{I_t}[\mu_{I_t}|\mathcal{E}] )^q}}{T(\Pr\{\mathcal{E}\})^q} \leq \frac{\sum_{t=\overline{T}+1}^T \E_{I_t}\left[\frac{1}{\left(\mu_{I_t}\right)^q}\bigg \vert \mathcal{E}\right]}{T(\Pr\{\mathcal{E}\})^q} = \frac{\E_{I_1,\ldots,I_t} \left[\sum_{t=\overline{T}+1}^T \frac{1}{\left(\mu_{I_t}\right)^q} \bigg\vert\mathcal{E}\right]}{T(\Pr\{\mathcal{E}\})^q}
\end{align*}

For simplicity, we will drop the subscripts. By reindexing the arms so that $\{1,2,\ldots,\ell\}$ are the arms pulled at least once in Phase II, and letting $m_i$ be the number of times (the reindexed) arm $i$ is pulled in Phase II, we have
\begin{align*}
    \frac{\E \left[\sum_{t=\overline{T}+1}^T \frac{1}{\left(\mu_{I_t}\right)^q} \bigg\vert\mathcal{E}\right]}{T(\Pr\{\mathcal{E}\})^q} & = \frac{\E \left[\sum_{i=1}^\ell m_i (\mu_i)^{-q} \bigg\vert\mathcal{E}\right]}{T(\Pr\{\mathcal{E}\})^q} 
      \leq \frac{\E \left[\sum_{i=1}^\ell m_i \left(\mu^*-4\sqrt{\frac{2\sigma^2 \log T}{n_i}}\right)^{-q}\bigg\vert\mathcal{E}\right]}{T(\Pr\{\mathcal{E}\})^q} \tag{via Lemma \ref{lem:suboptimal_arms:anytime}}\\
\end{align*}

Now, we split our analysis into two cases: when $p < -1$ and $-1\leq p< 0$ respectively. We will use two separate claims for these cases

\textit{Case 1: $p < -1$:} In this case, $q = |p_a| =|p|$. The following claim holds for this case (see Appendix \ref{proof:second} for proof).

\begin{restatable}{claim}{LemmaNumericSecond}
\label{lem:binomial_second}
For all $q \geq 1$ and reals $x \in \left[0, \frac{1}{2q}\right]$, we have $(1-x)^{-q} \leq  1 + 2qx$. 
\end{restatable}
Now, let $u=4\sqrt{\frac{2\sigma^2\log T}{n_i}}$. Then we have
\begin{align*}
u&=4\sqrt{\frac{2\sigma^2\log T}{n_i}}\leq 4\sqrt{\frac{4k\sigma^2\log T}{\overline{T}}}     \tag{$n_i\geq\frac{\overline{T}}{2k}$}\\
& =4\sqrt{\frac{4k\sigma^2(\mu^*)^2\log T}{512kp^2\sigma^2\log T}} = \frac{4\mu^*}{\sqrt {128p^2}} \tag{$\overline{T} = \frac{512k\sigma^2p^2\log T}{(\mu^*)^2}$}
\end{align*}
Thus, the quantity $x = \frac{u}{\mu^*} = \frac{4}{\sqrt{128}|p|} \leq \frac{1}{2q}$ (as $|p|=q$).

\textit{Case 2: $-1 \leq p < 0$:} For this case, $ p_a=1$; we modify the above claim slightly to get the following result (see Appendix \ref{proof:third} for proof).
\begin{restatable}{claim}{LemmaNumericThird}
\label{lem:binomial_third}
For all $0 < q \leq 1$ and reals $x \in \left[0, \frac{1}{2}\right]$, we have $(1-x)^{-q} \leq  1 + 2qx$. 
\end{restatable}
Again, let $u=4\sqrt{\frac{2\sigma^2\log T}{n_i}}$. Then we have
\begin{align*}
u&=4\sqrt{\frac{2\sigma^2\log T}{n_i}}\leq 4\sqrt{\frac{4k\sigma^2\log T}{\overline{T}}}     \tag{$n_i\geq\frac{\overline{T}}{2k}$}\\
& =4\sqrt{\frac{4k\sigma^2(\mu^*)^2\log T}{512k\sigma^2\log T}} = \frac{4\mu^*}{\sqrt {128}} \tag{$\overline{T} = \frac{512k\sigma^2\log T}{(\mu^*)^2}$}
\end{align*}
which implies, the quantity $x = \frac{u}{\mu^*} = \frac{4}{\sqrt{128}} \leq \frac{1}{2}$. Thus, we have, by Claim \ref{lem:binomial_second} and  \ref{lem:binomial_third} with $x = \frac{u}{\mu^*}$, $\forall \ p < 0$,
\begin{align*}
\frac{1}{y} & \leq \frac{\E \left[\sum_{i=1}^\ell m_i (\mu^*)^{-q} \left(1+\frac{4q}{\mu^*}\sqrt{\frac{2\sigma^2 \log T}{n_i}}\right)\bigg\vert\mathcal{E}\right]}{T(\Pr\{\mathcal{E}\})^q} \\
     & \leq \frac{\E \left[\sum_{i=1}^\ell m_i (\mu^*)^{-q} \left(1+\frac{4q}{\mu^*}\sqrt{\frac{2\sigma^2 \log T}{m_i}}\right)\bigg\vert\mathcal{E}\right]}{T(\Pr\{\mathcal{E}\})^q} \tag{since $n_i \geq m_i$}  \\
     & \leq \frac{\E \left[\sum_{i=1}^\ell (\mu^*)^{-q} \left(m_i +\frac{4q\sqrt{m_i}}{\mu^*}\sqrt{{2\sigma^2 \log T}}\right)\bigg\vert\mathcal{E}\right]}{T(\Pr\{\mathcal{E}\})^q}  
\end{align*}
Now, invoking Cauchy-Schwarz inequality on $\sum_{i=1}^\ell\sqrt{m_i}$ and using $m_i \leq T-\overline{T}$, we have
\begin{align}
     \frac{1}{y}& \leq \frac{\E \left[(\mu^*)^{-q} \left(T-\overline{T} +\frac{4q\sqrt{\ell(T-\overline{T})}}{\mu^*}\sqrt{{2\sigma^2 \log T}}\right)\bigg\vert\mathcal{E}\right]}{T(\Pr\{\mathcal{E}\})^q} \nonumber \\
     & \leq \frac{\E \left[(\mu^*)^{-q} \left(T +\frac{4q\sqrt{kT}}{\mu^*}\sqrt{{2\sigma^2 \log T}}\right)\bigg\vert\mathcal{E}\right]}{T(\Pr\{\mathcal{E}\})^q} \tag{$\ell \leq k$ and $T-\overline{T}\leq T$}\nonumber\\
     & = \frac{(\mu^*)^{-q}\left(1 + \frac{4q\sqrt{2k\sigma^2\log T}}{\sqrt{T}\mu^*}\right)T}{T(\Pr\{\mathcal{E}\})^q} \nonumber\\
     & \leq \frac{1}{\left((\mu^*)^q-\frac{4q\sqrt{2k\sigma^2\log T}(\mu^*)^{q-1}}{\sqrt T}\right){(\Pr\{\mathcal{E}\})^q}}, \label{ineq:boundy}
\end{align}
where inequality \eqref{ineq:boundy} holds using the standard inequality $(1+x) \leq \frac{1}{1-x} \ \forall \ 0 \le x \le 1$. Using the above inequality   and inequality \eqref{eqn:npreg_x_inv_bound}, we get
\begin{align*}
    \frac{1}{\frac{1}{x}+\frac{1}{y}} &\geq \frac{1}{\frac{\overline{T}k^q}{(\mu^*)^q T} + \frac{1}{\left((\mu^*)^q-\frac{4q\sqrt{2k\sigma^2\log T}(\mu^*)^{q-1}}{\sqrt T}\right){(\Pr\{\mathcal{E}\})^q}}} \\
    & = \frac{\left((\mu^*)^q-\frac{4q\sqrt{2k\sigma^2\log T}(\mu^*)^{q-1}}{\sqrt T}\right){(\Pr\{\mathcal{E}\})^q}}{1+\frac{\overline{T}k^q}{T}\left(1-\frac{4q\sqrt{2k\sigma^2\log T}}{\mu^*\sqrt T}\right){(\Pr\{\mathcal{E}\})^q}}
\end{align*}
Multiplying the numerator and denominator by $1-\frac{\overline{T}k^q}{T}\left(1-\frac{4q\sqrt{2k\sigma^2\log T}}{\mu^*\sqrt T}\right){(\Pr\{\mathcal{E}\})^q}$, we have
\begin{align}
    \frac{1}{\frac{1}{x}+\frac{1}{y}} &\geq \frac{\left((\mu^*)^q-\frac{4q\sqrt{2k\sigma^2\log T}(\mu^*)^{q-1}}{\sqrt T}\right){(\Pr\{\mathcal{E}\})^q}\left(1-\frac{\overline{T}k^q}{T}\left(1-\frac{4q\sqrt{2k\sigma^2\log T}}{\mu^*\sqrt T}\right){(\Pr\{\mathcal{E}\})^q}\right)}{1-\left(\frac{\overline{T}k^q}{T}\left(1-\frac{4q\sqrt{2k\sigma^2\log T}}{\mu^*\sqrt T}\right){(\Pr\{\mathcal{E}\})^q}\right)^2}\nonumber \\
    &\geq  {\left((\mu^*)^q-\frac{4q\sqrt{2k\sigma^2\log T}(\mu^*)^{q-1}}{\sqrt T}\right){(\Pr\{\mathcal{E}\})^q}\left(1-\frac{\overline{T}k^q}{T}\left(1-\frac{4q\sqrt{2k\sigma^2\log T}}{\mu^*\sqrt T}\right){(\Pr\{\mathcal{E}\})^q}\right)} \label{ineq:boundMean}
\end{align}
where the last inequality holds because of the denominator being less than 1.

We can expand Inequality \eqref{ineq:boundMean} to get
\begin{align*}
	\frac{1}{\frac{1}{x} + \frac{1}{y}} &\geq (\mu^*)^q \left(1-\frac{4q\sqrt{2k\sigma^2\log T}}{\mu^*\sqrt T}\right){(\Pr\{\mathcal{E}\})^q}\left(1-\frac{\overline{T}k^q}{T}\left(1-\frac{4q\sqrt{2k\sigma^2\log T}}{\mu^*\sqrt T}\right){(\Pr\{\mathcal{E}\})^q}\right)\\
    & \geq (\mu^*)^q \left(1-\frac{4q\sqrt{2k\sigma^2\log T}}{\mu^*\sqrt T}\right){(\Pr\{\mathcal{E}\})^q}\left(1-\frac{\overline{T}k^q}{T}\right) \tag{since $\left(1-\frac{4q\sqrt{2k\sigma^2\log T}}{\mu^*\sqrt T}\right){(\Pr\{\mathcal{E}\})^q} \leq 1$}\\
    & \geq (\mu^*)^q(\Pr\{\mathcal{E}\})^q\left( 1-\frac{4q\sqrt{2k\sigma^2\log T}}{\mu^*\sqrt T} -\frac{\overline{T}k^q}{T} \right) \tag{using $(1-x)(1-y) \geq (1-x-y) \ \forall \ x,y >0$} \\
    & = (\mu^*)^q(\Pr\{\mathcal{E}\})^q\left( 1-\frac{4q\sqrt{2k\sigma^2\log T}}{\mu^*\sqrt T} -\frac{128Sk^{q+1}}{T} \right) \\
    & = (\mu^*)^q(\Pr\{\mathcal{E}\})^q\left( 1-\frac{4q\sqrt{2k\sigma^2\log T}}{\mu^*\sqrt T} -\frac{512q^2k^{q+1}\sigma^2 \log T}{(\mu^*)^2 T}  \right)\tag{$q=|p|$}
\end{align*}
Exponentiating the last inequality by $\frac{1}{q}$, we have
\begin{align*}
\left(\frac{1}{\frac{1}{x} + \frac{1}{y}}\right)^{\tfrac{1}{q}}&\geq (\mu^*)(\Pr\{\mathcal{E}\})\left( 1-\frac{4q\sqrt{2k\sigma^2\log T}}{\mu^*\sqrt T} -\frac{512q^2k^{q+1}\sigma^2 \log T}{(\mu^*)^2 T}\right)^{\frac{1}{q}}\\
\end{align*}
Now, consider the following term 
\begin{align*}
    v& = \frac{4q\sqrt{2k\sigma^2\log T}}{\mu^*\sqrt T} + \frac{512q^2k^{q+1}\sigma^2 \log T}{(\mu^*)^2 T} \\
    & \leq \frac{4}{40k^{\frac{q}{2}}} +\frac{512}{1600} \tag{$\mu^*\geq\frac{40|p|\sigma k^{\frac{|p|+1}{2}}\sqrt{\log T}}{\sqrt T}$} \leq \frac{1}{2} 
\end{align*}
Thus, we can apply Claim \ref{lem:binomial} on $(1-v)^{\frac{1}{q}}$ to get
\begin{align*}
\left(\frac{1}{\frac{1}{x} + \frac{1}{y}}\right)^{\tfrac{1}{q}}&\geq (\mu^*)(\Pr\{\mathcal{E}\})\left( 1-\frac{8\sqrt{2k\sigma^2\log T}}{\ \mu^*\sqrt T} -\frac{1024q^2k^{q+1}\sigma^2 \log T}{q(\mu^*)^2 T}\right)
\end{align*}
Further, substituting $\Pr\{\mathcal{E}\} \geq 1 - \frac{2}{T}$, we have
\begin{align*}
\left(\frac{1}{\frac{1}{x} + \frac{1}{y}}\right)^{\frac{1}{q}} &\geq (\mu^*)\left(1-\frac{2}{T}\right)\left( 1-\frac{8\sqrt{2k\sigma^2\log T}}{\ \mu^*\sqrt T} -\frac{1024q^2k^{q+1}\sigma^2 \log T}{q(\mu^*)^2 T} \right)\\
&\geq (\mu^*)\left( 1-\frac{8\sqrt{2k\sigma^2\log T}}{\ \mu^*\sqrt T} -\frac{1024q^2k^{q+1}\sigma^2 \log T}{q(\mu^*)^2 T}  - \frac{2}{T} \right) \tag{using $(1-x)(1-y) \geq (1-x-y) \ \forall \ x,y >0$}
\end{align*}
Thus, the $q-$regret satisfies
\begin{align*}
	\mathrm{R}_T^q  &\leq \mu^* -  (\mu^*)\left( 1-\frac{8\sqrt{2k\sigma^2\log T}}{\ \mu^*\sqrt T} -\frac{1024q^2k^{q+1}\sigma^2 \log T}{q(\mu^*)^2 T} - \frac{2}{T} \right)\\
    & \leq \frac{8\sqrt{2k\sigma^2\log T}}{\sqrt T} + \frac{1024q^2k^{q+1}\sigma^2 \log T}{q(\mu^*) T}  + \frac{2\mu^*}{T} \\
    & \leq \frac{8\sqrt{2k\sigma^2\log T}}{\sqrt T} + \frac{256k^{\frac{q+1}{2}}\sigma \sqrt{\log T}}{10  \sqrt T}  + \frac{2\mu^*}{T}  \tag{$\mu^*\geq\frac{40{|p|}\sigma k^{\frac{|p|+1}{2}}\sqrt{\log T}}{\sqrt T}$ }\\
\end{align*}
and as a result, the $p-$mean regret satisfies
\begin{align}
    \mathrm{R}_T^p \leq  O\left(\frac{\sigma k^\frac{|p|+1}{2} \sqrt{\log T}}{\sqrt T}\right)\label{ineq:muGreater}
\end{align}
Thus, from inequalities \eqref{ineq:muLess} and \eqref{ineq:muGreater}, we get the final regret bound as
\begin{align*}
    \mathrm{R}_T^p \leq  O\left(\frac{\sigma k^\frac{|p|+1}{2} \sqrt{\log T}}{\sqrt T}\cdot\max(1,|p|)\right)
\end{align*}
which completes the proof of the theorem.
\end{proof}

\section{Auxillary Lemmas}
\label{sec:auxLem}
\begin{lemma}[Chernoff Bound]\label{lem:chernoff}
	Let $Z_1, \ldots, Z_n$ be independent Bernoulli random variables. Consider the sum $S = \sum_{r=1}^n Z_r$ and let $\nu = \E[S]$ be its expected value. Then, for any $\varepsilon \in [0,1]$, we have 
	\begin{align*}
		\prob \left\{ S \leq (1-\varepsilon) \nu \right\} & \leq \mathrm{exp} \left( -\frac{\nu \varepsilon^2}{2} \right), \text{  and} \\ 
		\prob \left\{ S \geq (1+\varepsilon) \nu \right\} & \leq \mathrm{exp} \left( -\frac{\nu \varepsilon^2}{3} \right).
	\end{align*}
\end{lemma}
\begin{lemma}[Hoeffding Inequality]\label{lem:hoeffding}
	Let $Z_1, \ldots, Z_n$ be independent random variables, with mean $\mu$ and subgaussianity parameter $\sigma$. Consider the empirical mean $\hat{\mu} = \frac{1}{n}\sum_{r=1}^n Z_r$. Then, we have 
	\begin{align*}
		\prob \left\{ |\hat{\mu}-\mu| \geq  \epsilon \right\} & \leq 2\exp \left( -\frac{n \epsilon^2}{2\sigma^2} \right)
	\end{align*}
\end{lemma}

\LemmaNumeric*


\begin{proof}
For $x\in[0,\tfrac12]$ we have
\[
-\ln(1-x)=\sum_{n\ge1}\frac{x^n}{n}\le\sum_{n\ge1}x^n=\frac{x}{1-x}\le 2x,
\]
so $\ln(1-x)\ge-2x$. Hence for $a\ge0$,
\[
(1-x)^a=\exp\big(a\ln(1-x)\big)\ge\exp(-2ax).
\]
Finally, using the standard inequality $e^{-y}\ge 1-y$ for all $y\in\mathbb R$ (the tangent-line bound for the convex exponential),
with $y=2ax\ge0$ we get $\exp(-2ax)\ge 1-2ax$. Combining completes the proof.
\end{proof}

\LemmaNumericSecond*
\begin{proof} \label{proof:second}
Define $\Phi(x) := \ln(1+2qx) + q\ln(1-x).$
The inequality is equivalent to $\Phi(x) \ge 0$, since
\begin{align*}
(1-x)^{-q} \le 1+2qx \iff -q\ln(1-x) \le \ln(1+2qx) \iff \Phi(x) \ge 0.
\end{align*}

Clearly, $\Phi(0)=0$.

Differentiating with respect to $x$, we have
\begin{align*}
\Phi'(x) = \frac{2q}{1+2qx} - \frac{q}{1-x}
= \frac{q \big(1 - 2x(1+q)\big)}{(1+2qx)(1-x)}.
\end{align*}
Since the denominator is positive on $\big[0, \tfrac{1}{2q}\big]$, the sign of $\Phi'(x)$ is determined by $1-2x(1+q)$, we have
\begin{align*}
\begin{cases}
\Phi'(x) \ge 0 & \text{for } x \le \tfrac{1}{2(1+q)}, \\[6pt]
\Phi'(x) \le 0 & \text{for } x \ge \tfrac{1}{2(1+q)}.
\end{cases}
\end{align*}
Thus $\Phi$ increases on $\left[0, \frac{1}{2(1+q)}\right]$ and decreases on $\left[\frac{1}{2(1+q)}, \frac{1}{2q}\right]$.  
Hence the minimum of $\Phi$ on $\left[0, \frac{1}{2q}\right]$ is attained at an endpoint, so
\begin{align*}
\Phi(x) \ge \min\{\Phi(0), \Phi(1/(2q))\} = \min\{0, \Phi(1/(2q))\}.
\end{align*}

Now evaluating at $x = \tfrac{1}{2q}$ we get
\begin{align*}
\Phi\!\left(\frac{1}{2q}\right) = \ln 2 + q \ln\!\left(1 - \frac{1}{2q}\right).
\end{align*}
So it suffices to show
\begin{align*}
q \ln\!\left(1 - \frac{1}{2q}\right) \ge -\ln 2,
\quad \text{i.e.}\quad
\left(1 - \frac{1}{2q}\right)^q \ge \frac{1}{2}.
\end{align*}

Define $\psi(q) := q \ln\!\left(1 - \frac{1}{2q}\right)$. Then
\[
\psi'(q) = \ln\!\left(1 - \frac{1}{2q}\right) + \frac{1}{2q\left(1-\frac{1}{2q}\right)}.
\]
Using the bound $\ln(1-t) \ge -\frac{t}{1-t}$ for $0 \le t < 1$ with $t=\tfrac{1}{2q}$, we obtain $\psi'(q) \ge 0$.  
Thus $\psi$ is increasing for $q \ge 1$. At $q=1$,
\begin{align*}
\psi(1) = \ln\!\left(\frac{1}{2}\right) = -\ln 2.
\end{align*}
So $\psi(q) \ge -\ln 2$ for all $q \ge 1$. Therefore,
\begin{align*}
\Phi\!\left(\frac{1}{2q}\right) = \ln 2 + \psi(q) \ge 0.
\end{align*}

Combining the above, we conclude $\Phi(x) \ge 0$ for all $x \in [0, \frac{1}{2q}]$.  
Exponentiation yields
\begin{align*}
(1-x)^{-q} \le 1+2qx.
\end{align*}
Hence, the claim is proved.

\end{proof}

\LemmaNumericThird*
\begin{proof}\label{proof:third}
Fix \(x\in\left[0,\tfrac12\right]\) and consider the function
\[
f(q):=(1-x)^{-q},\qquad q\in[0,1].
\]
A straightforward differentiation gives
\[
f'(q)=-\ln(1-x)\,(1-x)^{-q},\qquad
f''(q)=(\ln(1-x))^2(1-x)^{-q}\ge 0,
\]
so \(f\) is convex on \([0,1]\). By the convexity (the chord inequality) we have, for every \(q\in[0,1]\),
\[
f(q)\le (1-q)f(0)+q f(1) = (1-q)\cdot 1 + q\cdot(1-x)^{-1}
= 1 + q\big((1-x)^{-1}-1\big).
\]
Since \((1-x)^{-1}-1=\dfrac{x}{1-x}\) and \(1/(1-x)\le 2\) for \(x\in[0,\tfrac12]\), we obtain
\[
(1-x)^{-q} \le 1 + q\frac{x}{1-x} \le 1 + 2qx,
\]
which is the desired inequality.
\end{proof}

\end{document}